\documentclass[letterpaper]{article}
\pdfoutput=1 
\usepackage{aaai23}  
\usepackage{times}  
\usepackage{helvet}  
\usepackage{courier}  
\usepackage[hyphens]{url}  
\usepackage{graphicx} 
\usepackage{natbib}  
\usepackage{caption} 
\usepackage{algorithm}
\usepackage{algorithmic}
\usepackage{amssymb}
\usepackage{amsmath}
\usepackage{mathtools}
\usepackage{subcaption}
\usepackage{amsthm}
\usepackage{color}
\usepackage{booktabs}
\usepackage{hyperref}

\definecolor{commentcolor}{RGB}{110,154,155}   

\theoremstyle{plain}
\newtheorem{theorem}{Theorem}[section]
\newtheorem{proposition}[theorem]{Proposition}
\newtheorem{definition}[theorem]{Definition}
\usepackage{newfloat}
\usepackage{listings}
\DeclareCaptionStyle{ruled}{labelfont=normalfont,labelsep=colon,strut=off} 
\floatstyle{ruled}
\newfloat{listing}{tb}{lst}{}
\floatname{listing}{Listing}
\title{Robust Representation Learning by Clustering with Bisimulation Metrics \\for Visual Reinforcement Learning with Distractions}
\author{
	Qiyuan Liu\textsuperscript{\rm 1}, Qi Zhou\textsuperscript{\rm 1}, Rui Yang\textsuperscript{\rm 1}, Jie Wang\textsuperscript{\rm 1,2}\thanks{Corresponding author.}
}
\affiliations{
	\textsuperscript{\rm 1}CAS Key Laboratory of Technology in GIPAS, University of Science and Technology of China\\
 \textsuperscript{\rm 2}Institute of Artificial Intelligence, Hefei Comprehensive National Science Center\\
\{hy980505,zhouqida,yr0013\}@mail.ustc.edu.cn\\ jiewangx@ustc.edu.cn

}

\usepackage{bibentry}

\begin{document}
	
	\maketitle
	\begin{abstract}
		Recent work has shown that representation learning plays a critical role in sample-efficient reinforcement learning (RL) from pixels. 
		Unfortunately, in real-world scenarios, representation learning is usually fragile to task-irrelevant distractions such as variations in background or viewpoint.
		To tackle this problem, we propose a novel clustering-based approach, namely \textbf{C}lustering with \textbf{B}isimulation \textbf{M}etrics (CBM), which learns robust representations by grouping visual observations in the latent space.
		Specifically, CBM alternates between two steps: (1) grouping observations by measuring their bisimulation distances to the learned prototypes; (2) learning a set of prototypes according to the current cluster assignments. 
		Computing cluster assignments with bisimulation metrics enables CBM to capture task-relevant information, as bisimulation metrics quantify the behavioral similarity between observations. Moreover, CBM encourages the consistency of representations within each group, which facilitates filtering out task-irrelevant information and thus induces robust representations against distractions. 
		An appealing feature is that CBM can achieve sample-efficient representation learning even if multiple distractions exist simultaneously.
		Experiments demonstrate that CBM significantly improves the sample efficiency of popular visual RL algorithms and achieves state-of-the-art performance on both multiple and single distraction settings. The code is available at \href{https://github.com/MIRALab-USTC/RL-CBM}{https://github.com/MIRALab-USTC/RL-CBM}. 
	\end{abstract}
	
	\section{Introduction}Reinforcement learning (RL) from visual observations has achieved remarkable success in various domains ranging from video games \cite{human,lample2017playing} to robotics manipulation \cite{levine2016end,kalashnikov2018scalable}. Recently, representation learning that embeds image inputs into low-dimensional vectors has drawn great attention in visual RL. Popular approaches include learning sequential autoencoders \cite{sacae,slac}, applying data augmentation \cite{drq,drac} or constructing auxiliary tasks \cite{aux1,curl,atc}. However, recent work has shown that many of these methods can be easily distracted by task-irrelevant distractions, such as variations in background or viewpoint \cite{dcs,tia}. A fundamental challenge is to learn robust representations that effectively capture task-relevant information and are invariant to the task-irrelevant distractions \cite{dbc}.
	
	In order to learn robust representations, many recent approaches remedy the overfitting of the encoder by applying strong data augmentation, such as color-jitter \cite{rad} or random convolution \cite{randconv}.
	However, they require prior knowledge of the distraction to choose semantic-preserving image transformations in RL settings \cite{secant}. 
	Moreover, many of these methods require access to the environments without distractions for sample-efficient policy optimization. \cite{svea,soda}.
	Another line of work proposes a series of auxiliary tasks. One of the most popular schemes is to introduce contrastive objectives, which can encourage representations to preserve dynamics of the latent space \cite{dreaming,tpc} or behavioral similarity between states \cite{dbc,pse}. However, these contrastive-based methods usually rely on carefully constructed positive-negative pairs or a large batch size \cite{swav}.
	
	In this paper, we propose a novel clustering-based approach, namely \textbf{c}lustering with \textbf{b}isimulation \textbf{m}etrics (CBM), which learns robust representations by discriminating between groups of visual observations with similar behavior. 
	Specifically, at each training step, CBM first computes predicted cluster assignments via the geometric similarity between representations and a set of prototypes. 
	Then, CBM calculates target cluster assignments by measuring the bisimulation distances between visual observations and prototypes. 
	Finally, CBM optimizes the encoder and prototypical representations by enforcing the consistency between the predicted and target assignments. 
	An appealing feature of CBM is that it can improve the robustness against distractions without domain knowledge or carefully constructed pairwise comparison.
	Moreover, compared with previous clustering-based methods, CBM exploits the properties of RL tasks and achieves task-specific clustering, which enables CBM to extract task-relevant information effectively.
	
	The proposed method CBM is compatible with most visual RL algorithms. In our experiments, we combine CBM with two popular baselines, DrQ \cite{drq} and DrQ-v2 \cite{drqv2}. We perform empirical evaluations on Distracting Control Suite \cite{dcs}. Experiments demonstrate that our method significantly improves the performance of these two algorithms and achieves state-of-the-art performance both on multiple and single distraction settings.  
	We also conduct quantitative analysis and visualization of the clustering results to show that our method learns representations that hardly contain task-irrelevant information.
	
	To summarize, this paper makes the following contributions: 
	(1) We are the first to incorporate the task-specific properties of RL problems into clustering-based representation learning.
	(2) We propose a novel approach CBM, which learns robust representations by grouping visual observations with bisimulation metrics.
	(3) Our experiments demonstrate that the proposed CBM significantly improves performance over popular visual RL baselines and achieves sample-efficient representation learning even if multiple distractions exist simultaneously.
	
	\section{Related Work}
	\noindent \textbf{Clustering for representation learning.}
	Learning representations by clustering is one of the most promising approaches for self-supervised learning of neural networks. \citet{deepcluster} cluster deep features by k-means and use the cluster assignments as pseudo-labels to learn convnets. 
	To avoid degenerate solutions, \citet{selflabel} propose a principled formulation by adding the constraint that the labels must induce equipartition of the data. 
	Later, \citet{swav} learned prototypical representations through contrastive losses. They obtain online assignments and make it possible to scale to any dataset size. 
	Following this, recent approaches \cite{proto,dreamerpro} combine prototypical representations and dynamics learning in the RL setting. 
	In contrast to these methods, CBM incorporates the task-specific properties of RL problems into the clustering process and thus induces robust representations against distractions.
	
	\noindent \textbf{Representation learning for RL}
	Many recent methods have taken inspiration from the successes of representation learning in computer vision to improve sample efficiency and generalization \cite{vincent2008extracting,chen2020simple,xie2020unsupervised}.
	Learning autoencoders by minimizing reconstruction errors has been proven effective in visual RL \cite{slac,sacae,dreamer}.
	Later, Several researchers \citet{atc,mcurl} perform contrastive learning that maximizes agreement between augmented versions of the same observation. 
	More recently, adopting data augmentations such as random shifts\cite{drq,drqv2} has also been demonstrated effective in visual RL. 
	Another line of work \cite{pisac,infomax,virtual} has designed various auxiliary tasks to encourage the model to capture the predictive information. 
	However, when training with distractions, most of these methods suffer from performance degradation due to the interference of task-irrelevant information.
	
	\noindent \textbf{RL with visual distractions.}
	Learning control from pixels with distractions requires robust representation learning. Recent work has demonstrated the effectiveness of strong augmentation techniques in RL \cite{rad}. However, the strong augmentation may cause training instability or divergence.
	Recent approaches resolve this problem by policy distillation \cite{secant} or $Q$ function stabilization \cite{svea}.  
	Although their results are encouraging, these methods usually require domain knowledge to choose proper augmentation or clean environments without distractions for sample-efficient training. 
	In model-based RL, many papers \cite{tpc,dreaming} replace the reconstruction-based objectives with contrastive-based losses, which encourages the encoder to capture controllable and predictable information in the latent space. Another kind of method \cite{pse,dbc} proposes to learn invariant representations for RL by forcing the latent space to preserve the behavioral similarity. These methods rely on constructing positive pairs or learning accurate dynamics models, which may be difficult in environments with complex distractions.
	
	\section{Preliminaries }
	In this section, we first present the notation and the definition of bisimulation metrics. Then we briefly introduce the visual RL baselines and previous clustering-based representation learning methods that CBM builds on top of.
	\subsection{Notation}
	We consider the environment as an infinite horizon Markov decision process defined by a tuple $\langle \mathcal{S},\mathcal{A},\mathcal{P},r,\gamma\rangle$, where $\mathcal{S}$ is the state space, $\mathcal{A}$ is the action space, $\mathcal{P}(\cdot \mid \mathbf{s}, \mathbf{a})$ is the transitioning probability from state $\mathbf{s} \in \mathcal{S}$ to $\mathbf{s}^{\prime} \in \mathcal{S}$ , $r:\mathcal{S}\times \mathcal{A}\rightarrow [0,1]$ is a reward function, and $\gamma \in [0,1)$ is the discount factor. Since image observations only provide partial state observability, we define a state $\mathbf{s}_t$ by stacking several consecutive image observations $\mathbf{s}_t = \{\mathbf{o}_t,\mathbf{o}_{t-1},\ldots,\mathbf{o}_{t-k}\}$, $\mathbf{o}\in \mathcal{O}$, where $\mathcal{O}$ is the high-dimensional observation space (image pixels). The goal is to learn a parameterized policy $\pi_{\theta}: \mathcal{S}\rightarrow \mathcal{A} $ such that maximizes the cumulative return $\mathbb{E}_{\pi}\left[\sum_{t=0}^{\infty}\gamma^t r(\mathbf{s}_t, \mathbf{a}_t)\right]$.
	
	\subsection{Bisimulation Metrics}\label{bisim}
	Intuitively, two states are behaviorally equivalent if they receive indistinguishable reward sequences given any action. 
	By recursively utilizing the reward signal, bisimulation metrics \cite{ferns2011} define a distance function $d: \mathcal{S}\times\mathcal{S}\rightarrow \mathbb{R}_{\geq 0}$ that measures how behaviorally differently two states are. 
	\begin{definition}
		\cite{ferns2011} Given two arbitrary states $\mathbf{s}_1$,$\mathbf{s}_{2}\in\mathcal{S}$ and $c\in[0,1)$, the bisimulation distance between $\mathbf{s}_1$,$\mathbf{s}_{2}$ is defined as
		$$
		\begin{aligned}
			d(\mathbf{s}_1,\mathbf{s}_{2})=\max _{a \in \mathcal{A}}(&(1-c)\left|r_{\mathbf{s}_1}^{\mathbf{a}}-r_{\mathbf{s}_2}^{\mathbf{a}}\right|+\\
			&c W_{1}\left(\mathcal{P}(\cdot \mid \mathbf{s}_1, \mathbf{a}), \mathcal{P}\left(\cdot \mid \mathbf{s}_2, \mathbf{a}\right);d\right))
		\end{aligned}
		$$ \end{definition}
	\noindent The definition of bisimulation metrics consists of a distance between rewards and distance between state distributions. The latter is computed by Wasserstein metric $W_1$, which denotes the cost of transport mass from one distribution to another \cite{villani2021topics}.
	
	\subsection{Data-regularized Q for pixel-RL}
	Data-Regularized Q (DrQ) \cite{drq} uses an optimality invariant state transformation $f$, which preserves the $Q$ values to augment the training data when optimizing the $Q$ functions. DrQ generates $K$ samples for each raw image by applying random transformations and estimates the $Q$ function and $Q$ target by averaging over the augmented images. The addition of the augmentation method improves the performance of SAC \cite{sac} in DeepMind Control Suite \cite{dmc}.
	
	Recently, \citet{drqv2} propose DrQ-v2 that combines DrQ and TD3 \cite{td3}. DrQ-v2 replaces the maximum entropy term \cite{sac} with a scheduled noise for adjustable exploration and use target policy smoothing to reduce the bias of $\mathrm{Q}$ functions. Moreover, DrQ-v2 uses multi-step TD to learn value functions. Experiments show that DrQ-v2 achives more efficient learning than DrQ in complex visual control tasks.
	
	\subsection{SwAV and Proto-RL}\label{bg:proto}
	
	To learn visual features without supervision, \citet{swav} proposes an clustering-based method, Swapping Assignments between different views (SwAV), which computes cluster assignments in an online fashion and enforces consistency between cluster assignments produced for different augmentations of the same image.
	
	Specifically, given two image features $\mathbf{z}_s$ and $\mathbf{z}_t$ from two different augmentations of a same image, they compute their codes $\mathbf{q}_s$ and $\mathbf{q}_t$ by matching the features to $K$ learnable prototypes $\mathbf{C}=[\mathbf{c}_1,\ldots,\mathbf{c}_K]$. Then they setup a swapped prediction problem with the following loss function:
	\[L\left(\mathbf{z}_{t}, \mathbf{z}_{s}\right)=\ell\left(\mathbf{z}_{t}, \mathbf{q}_{s}\right)+\ell\left(\mathbf{z}_{s}, \mathbf{q}_{t}\right).\]where the $\ell$ represents the cross entropy loss between the code and the probability obtained by taking a softmax of the dot products of $\mathbf{z}$ and all prototypes.
	
	To ensure equal partitioning of the prototypes across all embeddings, given $B$ feature vectors $\mathbf{Z}=[\mathbf{z}_1,\ldots,\mathbf{z}_B]$, the codes $\mathbf{Q}=[\mathbf{q}_1,\ldots,\mathbf{q}_B]$ are calculated by 
	\begin{equation}\label{codes}
		\mathbf{Q}=\operatorname{Diag}(\mathbf{u}) \exp \left(\frac{\mathbf{C}^{\top} \mathbf{Z}}{\varepsilon}\right) \operatorname{Diag}(\mathbf{v}),
	\end{equation}
	where $\mathbf{u}\in\mathbb{R}^{K}$ and $\mathbf{v}\in\mathbb{R}^{B}$ are renormalization vectors, and $\varepsilon$ is a parameter that controls the smoothness of the codes. The renormalization vectors are computed using a small number of matrix multiplications using the iterative Sinkhorn-Knopp algorithm \cite{sinkhorn}.
	
	Proto-RL \cite{proto} draws inspiration from SwAV and adapts prototypical representation into RL setting. Different from SwAV, they enforce the consistency of cluster assignments from the predicted transition and the ground truth transition, which encourages the representations to preserve one-step dynamical similarity.
	They also demonstrate that the learned prototypes form the basis of latent space and thus induce efficient downstream exploration.

	\section{Methods}
	\begin{figure}[t]
		\vskip 0.1in
		\begin{center}
			\centerline{\includegraphics[width=1\columnwidth]{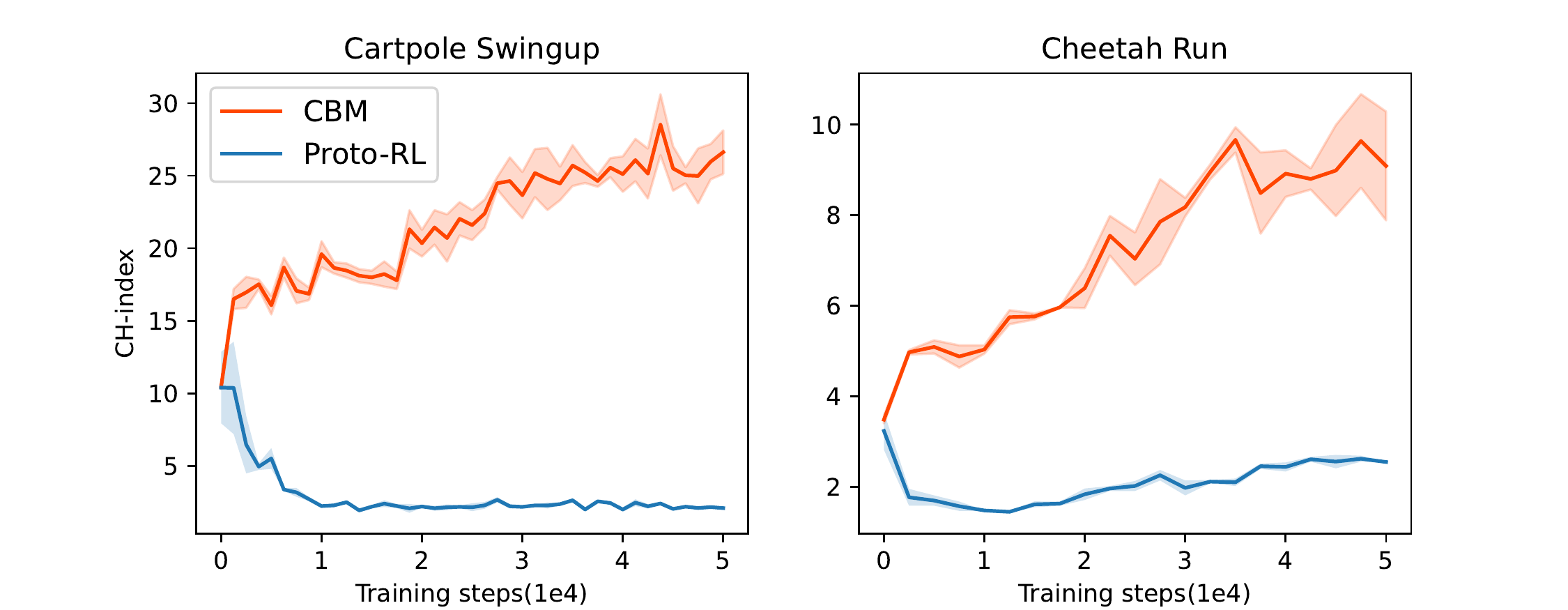}}
			\caption{Measure clustering quality with respect to physical states. The agents are trained with multiple distractions}
			\label{fig:cluster_analyze}
		\end{center}
		\vskip -0.4in
	\end{figure}
	In this section, we propose \textbf{c}lustering with \textbf{b}isimulation \textbf{m}etrics (CBM)---which groups observations in the latent space with bisimulation metrics---to learn robust representations against distractions. 
	Specifically, we first measure the clustering quality of clustering-based methods to study whether they capture task-relevant information and filter out task-irrelevant information. 
	Secondly, we propose to incorporate the task-specific properties in RL problems by introducing bisimulation metrics into the clustering process.
	Then, we propose an approach to approximate the bisimulation metrics in the clustering process. Finally, we describe the overall algorithm and implementation details of CBM.  
	
	\subsection{Measuring Clustering Quality in RL}
	\begin{figure*}[t]
		\centering
		\centerline{\includegraphics[width=0.95\textwidth]{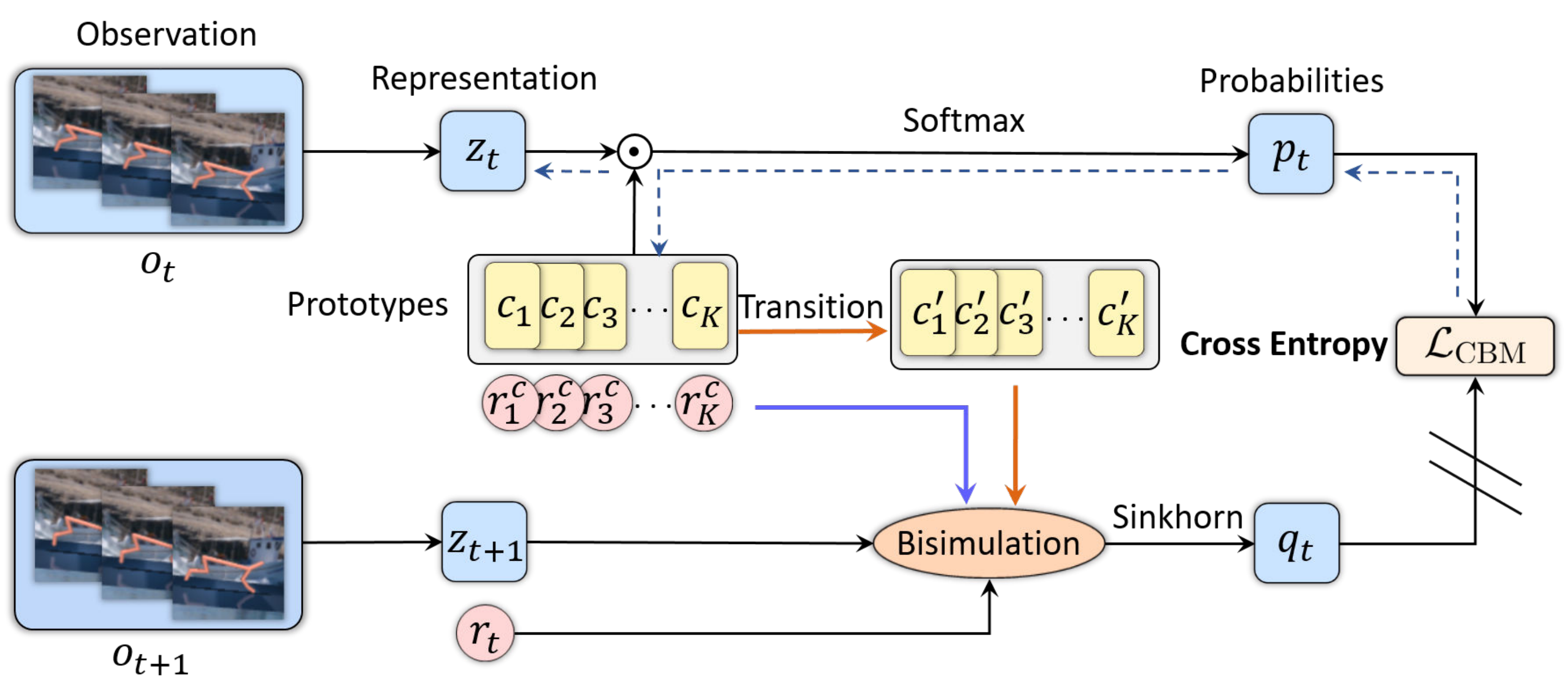}}
		\caption{
			The framework of CBM. The dashed line represents the backpropagation of gradients. CBM enforces consistency between predicted cluster assignments and target cluster assignments. We update prototypes' rewards and transitions by taking the average of batch rewards and using the learned transition model, respectively.}
		\label{fig:framework}
		\vskip -0.4cm
	\end{figure*}
	In this part, we take Proto-RL \cite{proto} for example to study whether previous clustering-based methods learn robust representations against distractions in visual RL tasks.
	A robust representation for visual RL should capture task-relevant information and be invariant to task-irrelevant distractions. To capture task-relevant information, clustering-based methods should discriminate observations that correspond to different physical states. Moreover, to filter out task-irrelevant information, clustering-based methods should assign observations corresponding to similar states to the same group and encode them as neighboring points in the latent space. Therefore, we can evaluate the robustness against distractions by the Calinski-Harabasz index (CH index) \cite{chindex} with respect to the low-dimensional physical states. 
	The CH index is the ratio of between-clusters dispersion and within-cluster dispersion. A higher value of the CH index indicates that the clusters are more dense and well separated.
	Specifically, we assign each observation to the cluster whose prototype is closest in the latent space. And we obtain the corresponding physical states from the simulator.
	
The results in Figure \ref{fig:cluster_analyze} show that Proto-RL struggles to learn robust representations while CBM achieves robustness against distractions. We observe that Proto-RL and CBM almost start from equal clustering quality because of the random initialization. However, as the training continues, CBM consistently improves the clustering quality in terms of states while Proto-RL sticks to a low value of CH-index. A potential reason is that the dynamical information used in Proto-RL to cluster observations contains task-irrelevant information about distractions and thus is not task-specific.
	
	\subsection{Clustering with Bisimulation Metrics}
In RL, reward signals provide essential task-specific information. As stated in Section \ref{bisim}, bisimulation metrics recursively utilize the reward signal  and quantify how behaviourally different two observations are. Therefore, we propose to calculate the cluster assignments according to the bisimulation metrics to the prototypes. The following proposition shows that a cluster formed in this manner will contain observations with close expected returns.
	\begin{proposition}
		Let $V^*$ be the optimal value function for a given discount factor $\gamma$. If $c\geq \gamma$, given the bisimulation metric $d$ and a prototype's underlying state $\mathbf{s}_c$, then for any two states $\mathbf{s}_1,\mathbf{s}_2$ such that $d(\mathbf{s}_1,\mathbf{s}_c)<\epsilon, d(\mathbf{s}_2,\mathbf{s}_c)<\epsilon$, we have
		$$|V^{*}(\mathbf{s}_1)-V^{*}(\mathbf{s}_2)|<\frac{2\epsilon}{1-c}.$$
	\end{proposition}
	\noindent Proof in Appendix. The proposition extends Theorem 5.1 in \citet{ferns2004metrics} to the clustering setting. It shows that if two states are close enough to the same prototype in terms of bisimulation distance, then their value difference is small. 
	In CBM, we calculate a bisimulation distance matrix $\mathbf{D}={(d_{ij})}_{K\times B}$ from the batch of observation encodings $\{\mathbf{z}_i\}_{i=1}^B$ and a set of prototypes $\{\mathbf{c}_j\}_{j=1}^K$, where $d_{ij} = d(\mathbf{z}_i, \mathbf{c}_j)$. Then we employ the Sinkhorn-Knopp clustering procedure to ensure equal partitioning of the prototypes across all embeddings and produce the target codes $\mathbf{Q}=[\mathbf{q}_1,\ldots,\mathbf{q}_B]$. Different from Equation \ref{codes}, we replace the dot product with the negative distance matrix.
	
	\begin{equation}\label{our_codes}
		\mathbf{Q}=\operatorname{Diag}(\mathbf{u}) \exp \left(\frac{-\mathbf{D}}{\varepsilon}\right) \operatorname{Diag}(\mathbf{v}).
	\end{equation}
	Consequently, we calculate soft cluster assignments to the set of prototypes, where we assign a high probability to the prototype that is behaviorally similar to the observation.
	
	\subsection{Approximating Bisimulation Metrics in CBM}
	
	\begin{figure*}[t]
		\centering
		\centerline{\includegraphics[width=0.9\textwidth]{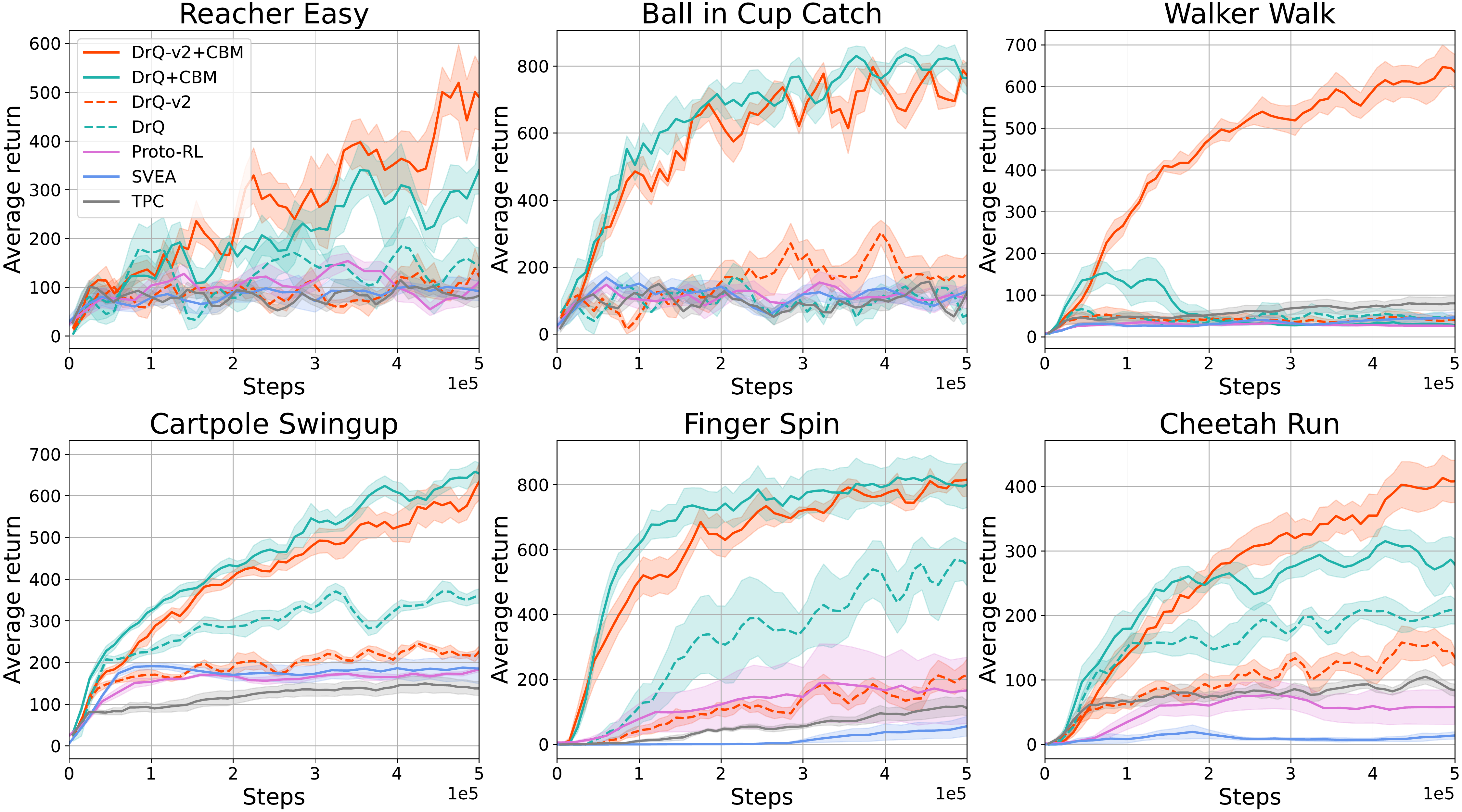}}
		\caption{Performance under multiple distractions. The solid curves correspond to the mean and the shaded region to the standard deviation. CBM significantly improves the sample efficiency for both DrQ and DrQ-v2.}
		\label{fig:easy}
		\vskip -0.4cm
	\end{figure*}
	In this part, we introduce the approximation to the bisimulation metrics between observations and prototypes. The bisimulation metrics consist of differences in rewards and transition. We can store the rewards and transitions of the observations in the replay buffer. However, the prototypes are representations in the latent space. Therefore, we cannot directly obtain their corresponding rewards and transitions from the environment.
	
	To tackle this problem, we first define the rewards of the prototypes. We approximate the prototypes' rewards by taking the weighted average over the rewards in a batch $\mathbf{r} = [r_1,\ldots,r_B]^{\top}$. 
	We estimate the prototype reward by
	\begin{equation}\label{rew-average}
		\hat{{r_k}}^{\mathbf{c}} = \frac{K}{B}\mathbf{w}_k^{\top}\mathbf{r}
	\end{equation}
	where the weight $\mathbf{w}_k=[q_{k1},\ldots,q_{kB}]^{\top}$ is taken from the target codes. 
	Since the target codes should satisfy $\sum_{j=1}^{B} q_{kj}=\frac{B}{K}$ after the Sinkhorn-Knopp clustering procedure, we multiply the weighted average by a factor $K/B$ for normalization. 
	At the beginning of the training, we initialize the prototype reward by randomly sampling from the replay buffer. We update the prototype reward at each training step by exponential moving average.
	\begin{equation}\label{rew-polyak}
		{r_k}^{\mathbf{c}}\leftarrow\beta \hat{{r_k}}^\mathbf{c}+(1-\beta) {r_k}^{\mathbf{c}}
	\end{equation}
	
	Secondly, to obtain the transitions of the prototypes, we train a deterministic latent dynamics model $\mathcal{P}$ concurrently. We obtain the prototype transition $\mathbf{c}_k^{\prime}$ from prediction of the model $\mathcal{P}$ given $\mathbf{c}_k$. Since prior work has shown that minimizing a quadratic loss is prone to collapse \cite{spr,deepmdp}, we train the latent dynamics model by one-step contrastive predictive coding (CPC) \cite{cpc}.
	
	Finally, we approximate the bisimulation distance $d(\mathbf{z}_i, \mathbf{c}_j)$ as
	$\|r_i-r_j^{\mathbf{c}}\| + \|\mathbf{z}_i^{\prime}-\mathbf{c}_j^{\prime}\|$ \cite{castro2020scalable} using the prototypes' rewards and transitions,  
	and thus we can compute the distance matrix $\mathbf{D}$ in Equation \ref{our_codes}. 
	\subsection{Algorithm}
In this part, we describe the training process of CBM. Our algorithm repeatedly performs the following steps.
Firstly, we sample a batch of data $\{\mathbf{o}_i,\mathbf{a}_i,r_{i},\mathbf{o}_{i}^{\prime},\}_{i=1}^B$ and map visual observations $\mathbf{o}_i,\mathbf{o}_{i}^{\prime}$ into latent states $\mathbf{z}_i,\mathbf{z}_{i}^{\prime}$ using the encoder.
Then, we predict the cluster assignment $\mathbf{p}_i$ by projecting $\mathbf{z}_i$ into a set of prototypes $\{\mathbf{c}_j\}_{j=1}^K$ and taking the softmax
\[\mathbf{p}_{i}=\frac{\exp \left(\frac{1}{\tau} \mathbf{z}_{i}^{\top} \mathbf{c}_{k}\right)}{\sum_{j=1}^K \exp \left(\frac{1}{\tau} \mathbf{z}_{i}^{\top} \mathbf{c}_{j}\right)},\]
where $\tau$ is a temperature parameter.
Next, we compute the target cluster assignment $\mathbf{q}_i$ by Equation \ref{our_codes}. To approximate the bisimulation metrics, we obtain prototypes' rewards by Equation \ref{rew-average}, \ref{rew-polyak} and transitions by using the latent dynamics model, respectively. 
Finally, we update the components in CBM. Specifically, we update prototypical representations $\{\mathbf{c}_j\}_{j=1}^K$ by optimizing the cross entropy loss 
\[\mathcal{L}_{\text{CBM}}=-\sum_{k=1}^B \mathbf{q}_{k} \log \mathbf{p}_{k}.\] 
We update the dynamics model by one-step contrastive predicting coding. We allow the gradients from $\mathcal{L}_{\text{CBM}}$ and dynamics loss to propagate to the encoder. Note that CBM is compatible with most visual RL algorithms. We illustrate the framework of CBM in Figure \ref{fig:framework}, and we provide the pseudocodes of CBM in the Appendix.
	\section{Experiments}
	In this section, we discuss the empirical results of CBM. We combine CBM with two popular algorithms, DrQ \cite{drq} and DrQ-v2 \cite{drqv2}, which improve SAC and TD3 for visual control tasks. To evaluate robustness against distractions, we test all methods on multiple and single distraction settings. Then, we analyze the impact of each component in CBM. Finally, we visualize the embedding space of CBM to demonstrate that CBM achieves task-specific clustering. All experiments report results over five seeds.
	
	\begin{figure*}[t]
		\centering
		\centerline{\includegraphics[width=1\textwidth]{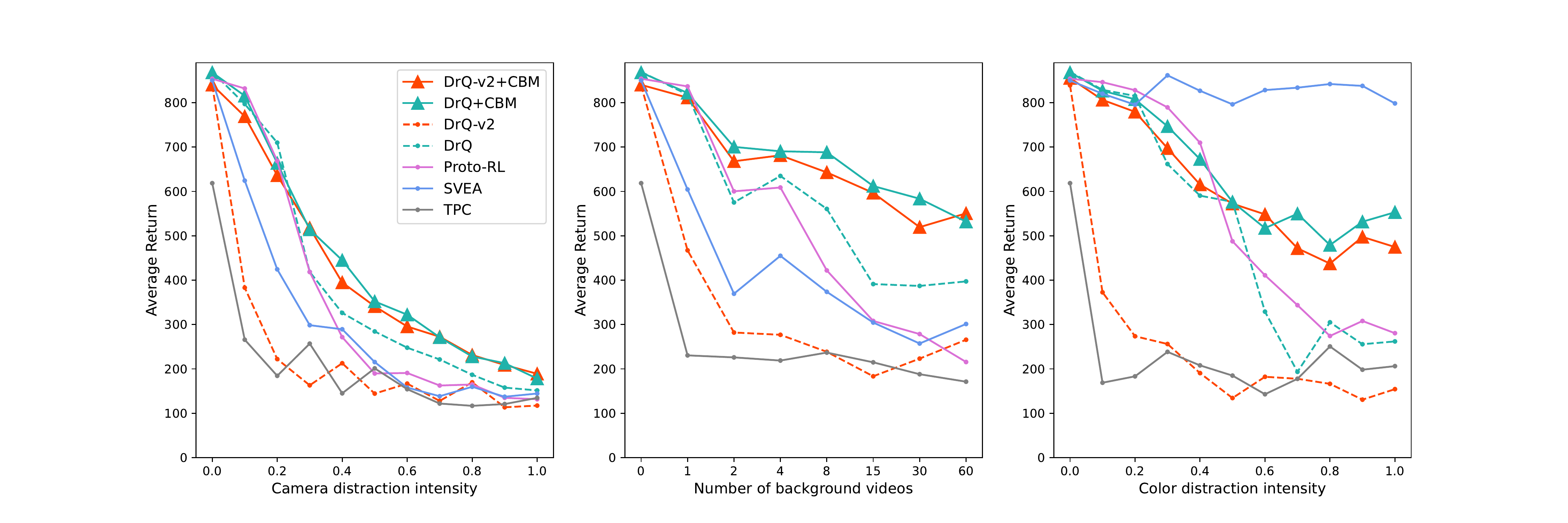}}
		\caption{Evaluation results after training with single distraction for 5e5 environment steps. CBM achieves state-of-the-art performance under the camera and background distraction. }
		\label{fig:single}
		\vskip -0.4cm
	\end{figure*}
	
	\noindent\textbf{Control Tasks} 
	We evaluate all agents on a challenging benchmark for vision-based control, Distracting Control Suite (DCS) \cite{dcs}. DCS extends DeepMind Control (DMC) \cite{dmc} with three kinds of visual distractions (background, color and camera pose). In multiple distractions settings, we use DCS "easy" setting, where the background images are sampled from 4 videos, and the scale for camera and color distraction is 0.1.
	
	\noindent\textbf{Baseline Methods} 
	Besides DrQ and DrQ-v2, we compare CBM with three state-of-the-art methods including Proto-RL \cite{proto}, which learn representations by grouping observations with dynamics information, SVEA \cite{svea}, which regularizes the representations by utilizing strong augmentation like random convolution, TPC \cite{tpc}, which contrastively learns a latent model for planning.
	
	\noindent\textbf{Architectures} When combining CBM with DrQ and DrQ-v2, we use a shared convolutional encoder followed with a linear projector with 50 outputs. CBM learns K=128 prototypes, each parameterized as a 50-dimensional vector.
	
	\subsection{Learning Control with Multiple Distractions}\label{sec:exp-multi-distract}
	In this part, we evaluate methods under multiple distractions in terms of sample efficiency and asymptotic performance. \citet{dcs} find that multiple distractors have a compounding effect: combined distractors degrade performance more than individually. The results in Figure \ref{fig:easy} show that CBM significantly improves the sample efficiency and asymptotic performance of DrQ and DrQ-v2. Moreover, CBM outperforms all baselines in six environments. In some tasks like finger spin or cheetah run, DrQ and DrQ-v2 are able to learn reasonable behaviour, but they require more samples to train the agent. The poor performance of Proto-RL, SVEA, and TPC indicate clustering with dynamics information, utilizing strong augmentation and learning latent dynamics contrastively struggle to overcome the challenging compounding effect of the multiple distractions. 
	\subsection{Learning Control with Single Distractions}
	In this part, we train agents with each single distraction in different magnitude and evaluate them at the end of training. We present the results in Figure \ref{fig:single}. CBM significantly improves the performance of both DrQ and DrQ-v2, especially under strong distractions. 
	We observe that Proto-RL performs fairly well when the distraction is mild, while its performance declines dramatically as the distraction scale increases. 
	This result indicates that clustering with dynamics information can help to handle a mild distraction but performs poorly under stronger distraction. 
	TPC only achieves about 600 average returns in the standard setting without distractions. We argue that its poor performance partly comes from the low sample efficiency of model training. 
	SVEA is quite robust to color distraction. 
	However, it suffers from severe degradation under the camera and background distraction. A potential reason is that the particular augmentation choice in SVEA (random convolution) can simulate the color variations well but can not approximate the camera and background variations. Therefore, we conclude that the data augmentation method can perform well with proper augmentation choice according to the domain knowledge but perform much worse than our method CBM when such domain knowledge is missing.
	\begin{figure}[t]
		\vskip 0.1in
		\begin{center}
			\centerline{\includegraphics[width=1\columnwidth]{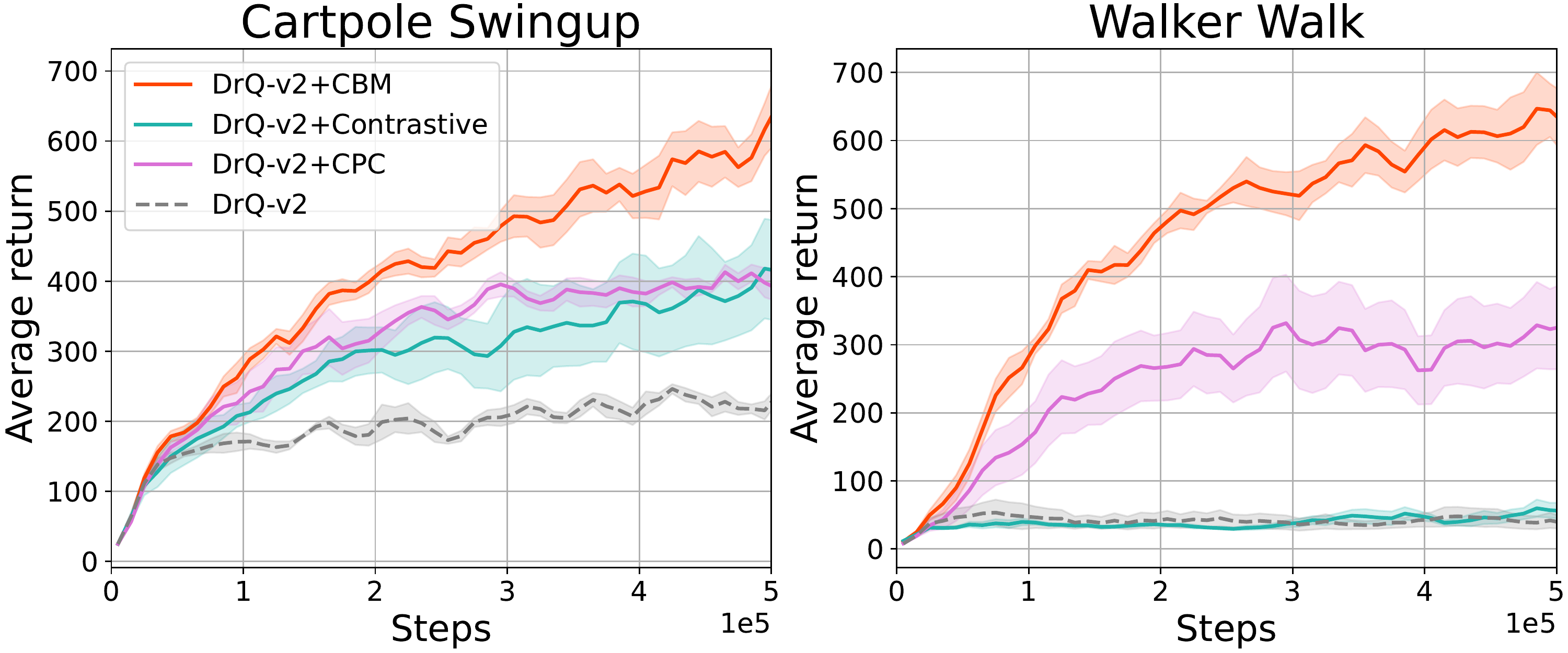}}
			\caption{CBM works better than training a latent model alone or contrastively utilizing the bisimulation metrics.}
			\label{ablation}
		\end{center}
		\vskip -0.3in
	\end{figure}
	\subsection{Ablation Study}\label{sec:abl}
	\begin{figure*}[t]
		\centering
		\centerline{\includegraphics[width=1\textwidth]{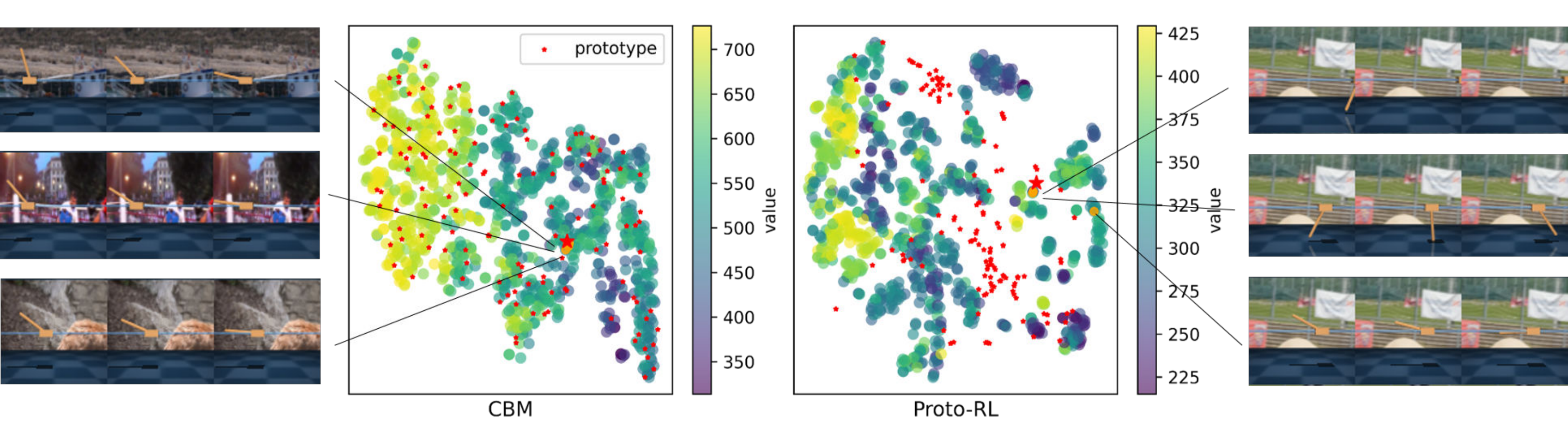}}
		\caption{t-SNE visualization of representations learned by CBM and Proto-RL. The color represents the state value, the red points represent the prototypes. We show an example prototype and three encodings that have biggest assignment probability to it. The cluster formed by CBM is more behaviorally similar and the prototypes disperse more uniformly in the latent space.}
		\label{fig:tsne&clu}
		\vskip -0.4cm
	\end{figure*}
	\begin{figure}[!ht]
		\vskip 0.05in
		\begin{center}
			\centerline{\includegraphics[width=1\columnwidth]{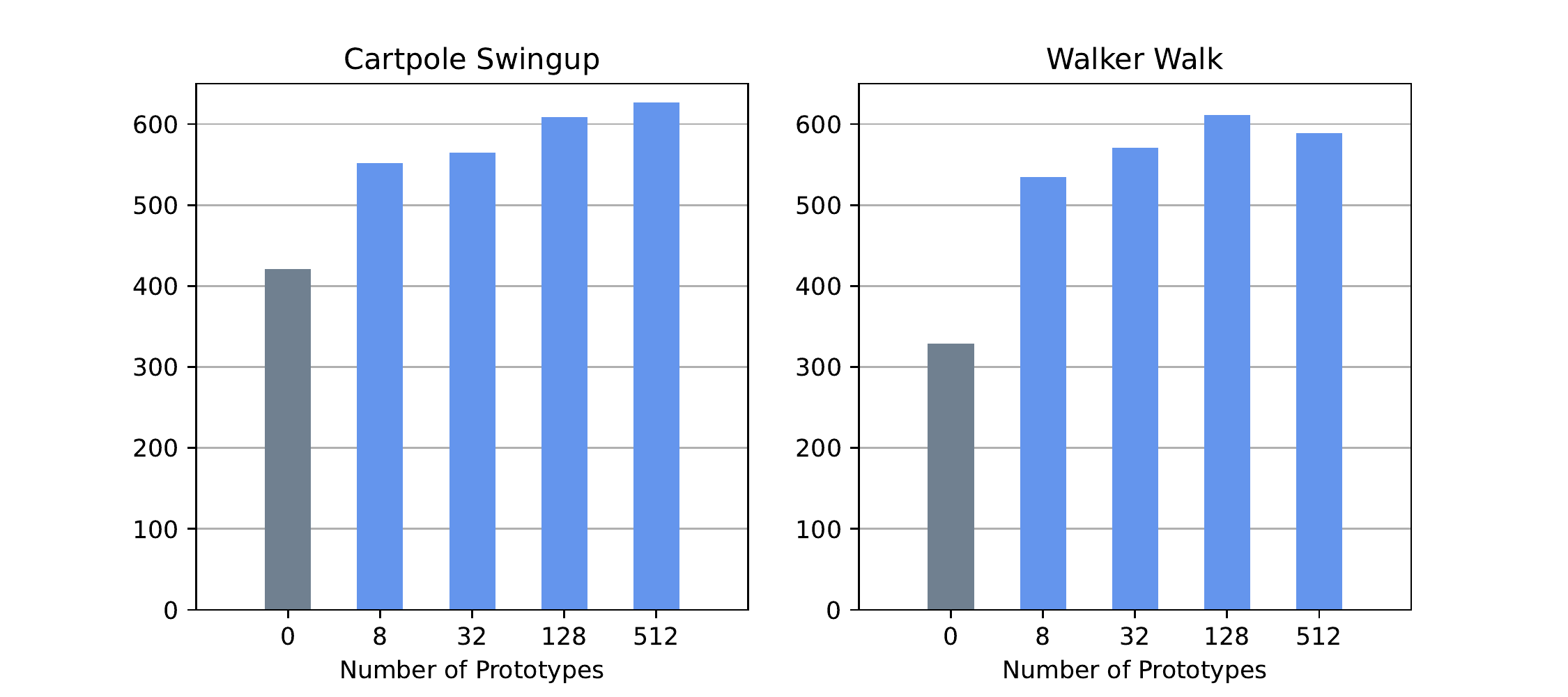}}
			\caption{Performance of CBM with various number of prototypes. Zero prototypes represent the implementation that only train the transition model.}
			\label{proto_k}
		\end{center}
		\vskip -0.3in
	\end{figure}
	This part analyzes the impact of clustering-based objectives, the transition model, and the number of prototypes. We conduct the experiments on cartpole swingup and walker walk in the multiple distraction setting.
	
	\noindent\textbf{Comparison with contrastive objectives} 
	We compare the clustering-based objective in CBM with a contrastive-based objective. Concretely, for every observation encoding, we select the positive sample as the closest observations from the sampled batch with respect to bisimulation metrics. We keep all hyperparameters fixed and approximate the bisimulation metrics in the same way as in CBM. We optimize the encoder by the InfoNCE loss \cite{cpc} and the dynamics loss. As shown in Figure \ref{ablation}, we observe that the contrastive-based method performs poorly on the multiple distraction setting, while CBM exploits the bisimulation metrics better and is much more sample-efficient. A potential reason is that the contrastive-based method is sensitive to negative samples and requires a large batch size.
	
	\noindent\textbf{Impact of transition model} To approximate the bisimulation metrics, we concurrently train a latent model by one-step CPC. Previous work has shown that learning latent dynamics can help accelerate representation learning and discriminate task-relevant information \cite{tpc}. To study the separate effect of training a latent model in our method, we update the encoder by training the transition model alone. From Figure \ref{ablation}, we investigate that training the latent model improves the performance of DrQ-v2. Adding the clustering loss $\mathcal{L}_{\text{CBM}}$ can further improve the sample efficiency and asymptotic performance, which proves the effectiveness of the clustering.
	
	\noindent\textbf{Impact of number of prototypes} We test CBM with a different number of prototypes ranging from 0 to 512. From Figure \ref{proto_k}, we find that even a small number of prototypes, such as eight, can significantly improve performance. We also observe that further increasing $K$ to 128 can make a slight improvement, while a larger $K$ may not always bring further improvement. These results suggest that the number of prototypes has little impact as long as it is "enough". Since using more prototypes increases the computational cost, we recommend training CBM with 128 prototypes.
	
	\subsection{Visualization Analysis}
	
	In this part, we visualize the embeddings trained with background distraction on cartpole swingup. We apply t-SNE algorithms on randomly selected observation encodings and the set of prototypical representations. Figure \ref{fig:tsne&clu} shows that CBM maps observations with similar values to neighboring regions while Proto-RL does not. This indicates that CBM effectively captures task-relevant information.
	
	To visualize the clustering results, we select a prototype and present the three nearest observations in terms of cosine similarity. We observe that observations obtained from the embedding of CBM exhibit similar behavior, while those obtained from Proto-RL behave quite differently. This observation demonstrates that CBM achieves task-specific clustering. We put more clustering results of CBM from different environments in the Appendix.
	
	We further color points in the t-SNE plot with their background videos.
The figure is shown in Appendix D. 
We find that CBM encodes visual observations with similar backgrounds uniformly in the latent space, while Proto-RL embeds them into neighboring regions. The results prove that CBM filters out task-irrelevant information.
	
	\begin{figure}[t]
		\vskip 0.1in
		\begin{center}
			\centerline{\includegraphics[width=0.9\columnwidth]{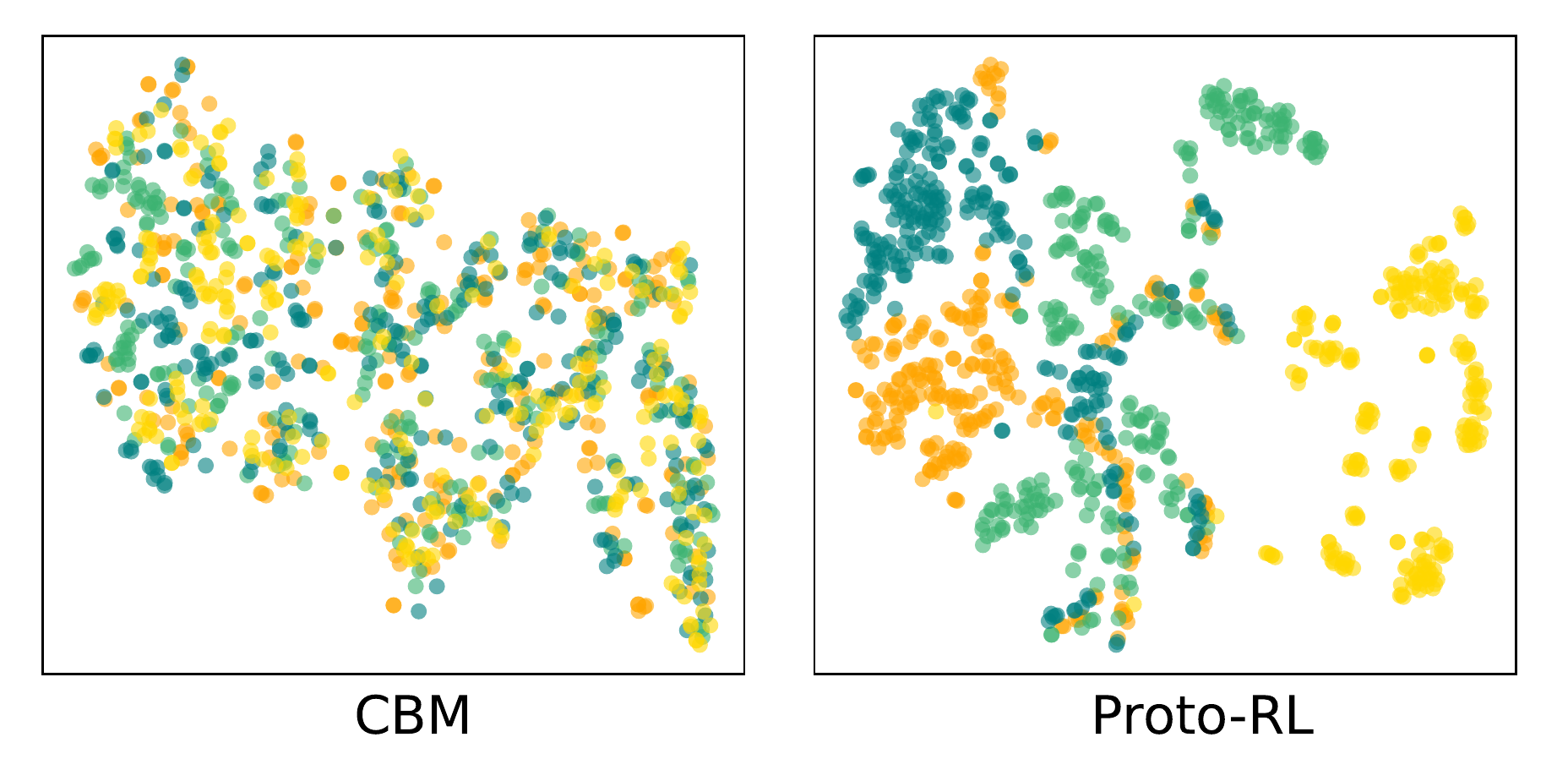}}
			\caption{t-SNE visualization of representations learned by CBM and Proto-RL. The color represents different background videos of the observations. }
			\label{tsne_bg}
		\end{center}
		\vskip -0.35in
	\end{figure}
	\section{Conclusion}
	Learning robust representations is critical for sample-efficient reinforcement learning from images. In this paper, we propose CBM, a novel clustering-based method that learns robust representations by grouping visual observations with bisimulation metrics. By incorporating the properties of RL tasks, CBM effectively captures task-relevant information and filters out task-irrelevant information. Experiments demonstrate that CBM significantly improves the sample efficiency of popular visual RL algorithms and achieves state-of-the-art performance on both multiple and single distraction settings.
\section{Acknowledgments}
We would like to thank all the anonymous reviewers for their insightful comments. This work was supported in part by National Science Foundations of China grants U19B2026, U19B2044, 61836006, and 62021001, and the Fundamental Research Funds for the Central Universities grant WK3490000004.
	\bibliography{aaai23}

\newpage
\appendix

\section{Propositions}
\begin{proposition}
Let $V^*$ be the optimal value function for a given discount factor $\gamma$. If $c\geq \gamma$, given the bisimulation metric $d$ and a prototype's underlying state  $\mathbf{s}_c$, then for any two states $\mathbf{s}_1,\mathbf{s}_2$ such that  $d(\mathbf{s}_1,\mathbf{s}_c)<\epsilon, d(\mathbf{s}_2,\mathbf{s}_c)<\epsilon$, we have
$$|V^{*}(\mathbf{s}_1)-V^{*}(\mathbf{s}_2)|<\frac{2\epsilon}{1-c}.$$
\end{proposition}
\begin{proof}
The proof uses techniques from the proof of Theorem 5.1 in \citet{ferns2004metrics}, adapting them to the clustering setting considered in this paper. 

For any $\mathbf{s},\mathbf{s}^{\prime}\in \mathcal{S}$, we suppose that $V_{0}(\mathbf{s})=0$ and $d(\mathbf{s},\mathbf{s}^{\prime})=0$, We define $V_{i+1}(\mathbf{s})$ and $d_{i+1}\left(\mathbf{s}, \mathbf{s}^{\prime}\right)$ by 
$$
V_{i+1}(\mathbf{s})=\max _{\mathbf{a} \in \mathcal{A}}(r_{\mathbf{s}}^{\mathbf{a}}+\gamma \sum_{u \in \mathcal{S}} P_{\mathbf{s} \mathbf{u}}^{\mathbf{a}} V_{i}(\mathbf{u})),
$$
$$d_{i+1}\left(\mathbf{s}, \mathbf{s}^{\prime}\right)=\max_{\mathbf{a}\in\mathcal{A}}(1-c)\left|r_{\mathbf{s}}^{\mathbf{a}}-r_{\mathbf{s}^{\prime}}^{\mathbf{a}}\right|+c W_1\left(\mathcal{P}_{\mathbf{s}}^{\mathbf{a}}, \mathcal{P}_{\mathbf{s}^{\prime}}^{\mathbf{a}};d_i\right).$$
We first prove $(1-c)\left|V_{i}({\mathbf{s}})-V_{i}\left({\mathbf{s}}^{\prime}\right)\right|\leq d_{i}\left({\mathbf{s}}, {\mathbf{s}}^{\prime}\right)$ by induction. With the induction hypothesis and the assumptions $c\geq \gamma, r_{\mathbf{s}}^{\mathbf{a}}\in [0,1]$, we have
$$\begin{aligned}
	&\frac{(1-c)\gamma}{c}V_{i}({\mathbf{s}})-\frac{(1-c)\gamma}{c}V_{i}({\mathbf{s}^{\prime}})\\\leq&(1-c)\left|V_{i}({\mathbf{s}})-V_{i}\left({\mathbf{s}}^{\prime}\right)\right|\leq d_{i}\left({\mathbf{s}}, {\mathbf{s}}^{\prime}\right).
	\end{aligned}$$
Then for any $\mathbf{s},\mathbf{s}^{\prime}\in\mathcal{S}$, we have
\[\begin{aligned}
    &(1-c)\left|V_{i+1}({\mathbf{s}})-V_{i+1}\left({\mathbf{s}}^{\prime}\right)\right| 
    \\=&(1-c)|\max_{\mathbf{a} \in \mathcal{A}}(r_{{\mathbf{s}}}^{\mathbf{a}}+\gamma \sum_{{\mathbf{u}} \in \mathcal{S}} P_{{\mathbf{s}} {\mathbf{u}}}^{\mathbf{a}} V_{i}({\mathbf{u}}))-
    \\&\qquad\quad\max_{\mathbf{a} \in \mathcal{A}}(r_{{\mathbf{s}}^{\prime}}^{\mathbf{a}}+\gamma \sum_{{\mathbf{u}} \in \mathcal{S}} P_{{\mathbf{s}}^{\prime} {\mathbf{u}}}^{\mathbf{a}} V_{i}({\mathbf{u}}))| 
       \\ \leq& (1-c) \max_{\mathbf{a} \in \mathcal{A}}|r_{{\mathbf{s}}}^{\mathbf{a}}-r_{{\mathbf{s}}^{\prime}}^{\mathbf{a}}
       +\gamma \sum_{{\mathbf{u}} \in \mathcal{S}}(P_{{\mathbf{s}} {\mathbf{u}}}^{\mathbf{a}}-P_{{\mathbf{s}}^{\prime} {\mathbf{u}}}^{\mathbf{a}}) V_{i}({\mathbf{u}})| 
       \\ \leq &\max_{\mathbf{a} \in \mathcal{A}}((1-c)|r_{{\mathbf{s}}}^{\mathbf{a}}-r_{{\mathbf{s}}^{\prime}}^{\mathbf{a}}|+
       \\ &\qquad \quad c|\sum_{{\mathbf{u}} \in \mathcal{S}}\left(P_{{\mathbf{s}} {\mathbf{u}}}^{\mathbf{a}}-P_{{\mathbf{s}}^{\prime} {\mathbf{u}}}^{\mathbf{a}}\right) V_{i}({\mathbf{u}}) \cdot \frac{(1-c) \gamma}{c}|) 
       \\ \leq &\max_{\mathbf{a} \in \mathcal{A}}((1-c) |r_{{\mathbf{s}}}^{\mathbf{a}}-r_{{\mathbf{s}}^{\prime}}^{\mathbf{a}}|+
       c W_{1}\left(P_{{\mathbf{s}}}^{\mathbf{a}}, P_{{\mathbf{s}}^{\prime}}^{\mathbf{a}};d_{i}\right)) 
       \\=&d_{i+1}\left({\mathbf{s}}, {\mathbf{s}}^{\prime}\right). \end{aligned}\]
  The last inequality comes from the dual representation of $W_1$ \cite{villani2021topics}. From the convergence of value iteration \cite{bellman1957markovian} and the formulation of bisimulation metrics \cite{ferns2004metrics}, we take limits on both sides of  $(1-c)\left|V_{i}({\mathbf{s}})-V_{i}\left({\mathbf{s}}^{\prime}\right)\right|\leq d_{i}\left({\mathbf{s}}, {\mathbf{s}}^{\prime}\right)$ and it comes to
  $$(1-c)|V^{*}(\mathbf{s})-V^{*}(\mathbf{s}^{\prime})|\leq d(\mathbf{s},\mathbf{s}^{\prime}).$$
  Next, from the assumptions $d(\mathbf{s}_1,\mathbf{s}_c)<\epsilon, d(\mathbf{s}_2,\mathbf{s}_c)<\epsilon$, we can see that
$$\begin{aligned}
(1-c)|V^{*}(\mathbf{s}_1)-V^{*}(\mathbf{s}_c)|\leq d(\mathbf{s}_1,\mathbf{s}_c)<\epsilon,\\
(1-c)|V^{*}(\mathbf{s}_2)-V^{*}(\mathbf{s}_c)|\leq d(\mathbf{s}_2,\mathbf{s}_c)<\epsilon.
\end{aligned}$$Finally, using the triangle inequality we have
\begin{align*}
&|V^{*}(\mathbf{s}_1)-V^{*}(\mathbf{s}_2)|
\\=&|V^{*}(\mathbf{s}_1)-V^{*}(\mathbf{s}_c)+V^{*}(\mathbf{s}_c)-V^{*}(\mathbf{s}_2)|
\\\leq& |V^{*}(\mathbf{s}_1)-V^{*}(\mathbf{s}_c)| + |V^{*}(\mathbf{s}_2)-V^{*}(\mathbf{s}_c)|
\\<&\frac{2\epsilon}{1-c},
\end{align*}
which completes the proof.
\end{proof}
\section{Experiments}
\subsection{Measuring Clustering Quality by CH index}
The Calinski-Harabasz index \cite{chindex} is the ratio of between-clusters dispersion and inter-cluster dispersion for all clusters,  where dispersion is defined as the sum of distances squared. The CH index is higher when clusters are dense and well separated, which relates to a standard concept of a cluster. In Section 4.1, we assign each observation encoding to the cluster with the nearest prototype. Given the corresponding physical states $[\mathbf{x}_1,\mathbf{x}_2,\ldots,\mathbf{x}_N]$, CH index for $K$ number of clusters is computed by
$$\frac{\sum_{k=1}^{K} n_{k}\left\|\mathbf{C}_{k}-\mathbf{C}\right\|^{2} }{\sum_{k=1}^{K} \sum_{i=1}^{n_{k}} \left\|\mathbf{x}_{i}-\mathbf{C}_{k}\right\|^{2}}\cdot\frac{N-K}{K-1},
$$
where $n_k$ is the number of observations in cluster $k$, $\mathbf{C}_{k}$ is the centroid of cluster $k$, and $\mathbf{C}$ is the centroid of all physical states data. 

\subsection{Distracting Control Suite Setting}
We evaluate all agents in Distracting Control Suite (DCS) \cite{dcs}, which extends DeepMind Control (DMC) \cite{dmc} with three kinds of visual distractions (variations in background, color and camera pose).  Figure \ref{dcs} shows the snapshots of all six environments. In these tasks, robots face multiple distractions at the same time. We use the static setting where each distraction is randomly sampled at the beginning of each episode and then fixed throughout the whole episode. The static setting is more challenging than the dynamic setting since the agents see less variety of distractions during training \cite{dcs}. We evaluate an agent by computing an average return over ten episodes after every 10K environment steps.
In the multiple distractions setting, we use DCS easy setting. Specifically, the background image is sampled from 4 videos, and the difficulty scales for camera and color distractions are set to $0.1$. 
The action repeat of each task is adopted from PlaNet \cite{plannet} as shown in Table \ref{action_repeat}, which is the common setting in visual RL.

\begin{figure}[t]
\vskip 0.2in
\begin{center}
\centerline{\includegraphics[width=1\columnwidth]{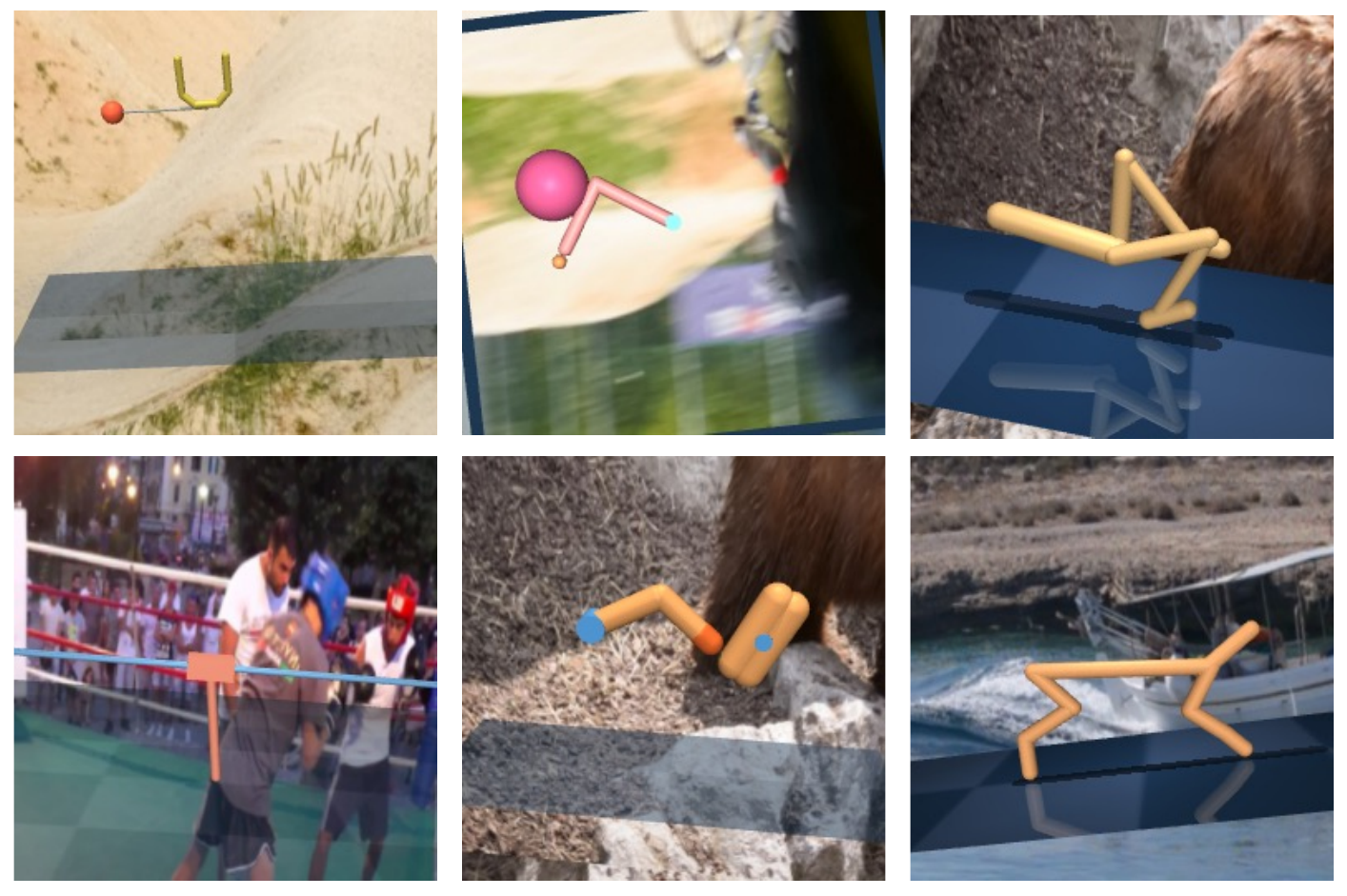}}
\caption{Six tasks in DCS. Agents face the variations in the camera pose, color and background simultaneously.}
\label{dcs}
\end{center}
\vskip -0.2in
\end{figure}
\begin{figure*}[t]
\centering
\centerline{\includegraphics[width=0.7\textwidth]{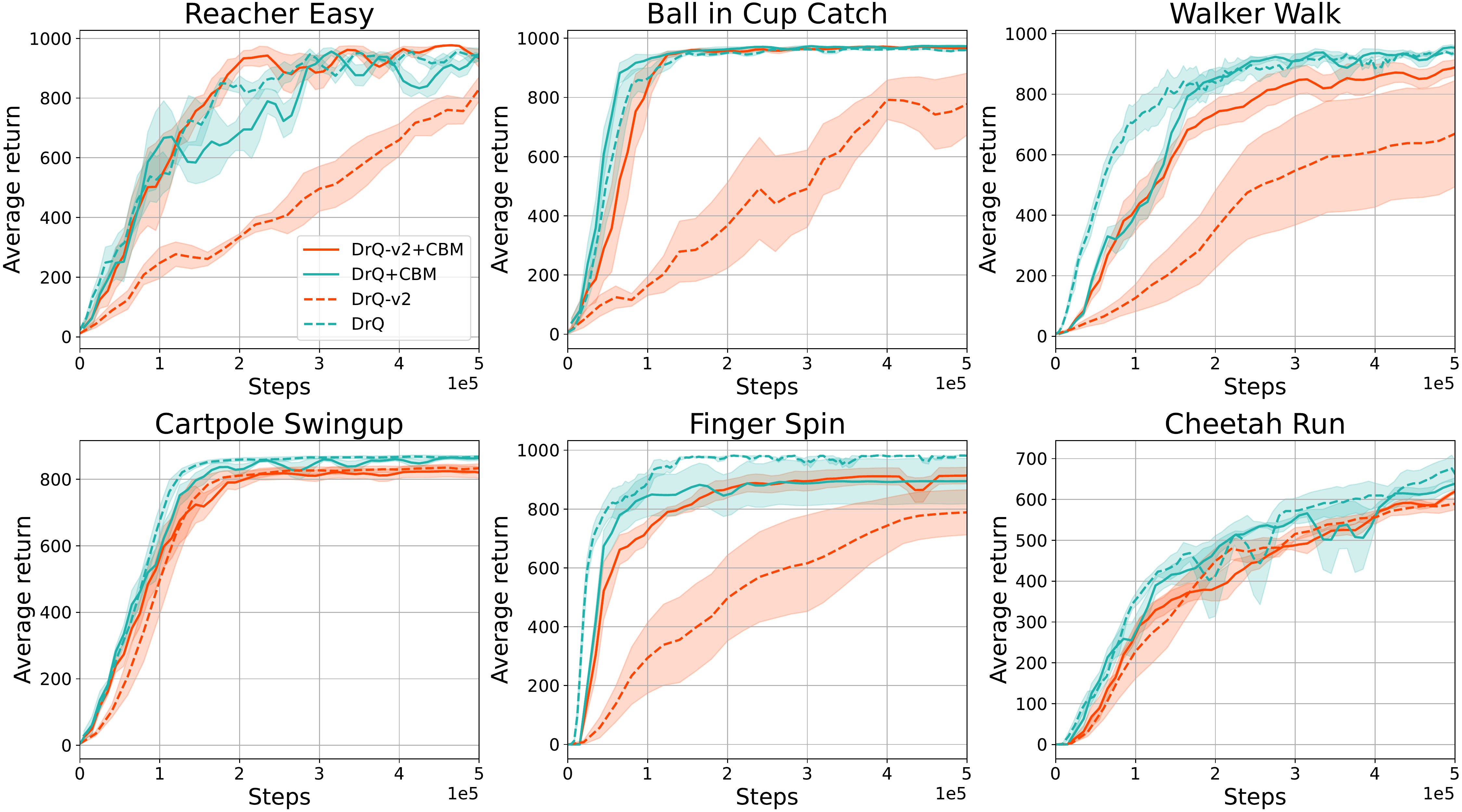}}
\caption{Results in DeepMind Control. CBM does not decrease the sample efficiency and asymptotic performance and even bring a further improvement on the default DMC setting without distractions.}
\label{fig:stand}
\vskip -0.05cm
\end{figure*}
\begin{table}[ht]
    \centering
    \begin{tabular}{cc}
    \toprule
        \textbf{Environment} & \textbf{Action repeat}\\ \midrule
        Ball in Cup-Catch & 4\\ 
        Cartpole Swingup & 8\\ 
        Cheetah Run & 4\\ 
        Finger Spin & 2\\ 
        Reacher Easy & 4\\ 
        Walker Walk & 2\\ \bottomrule
    \end{tabular}
    \caption{The action repeat hyperparameter for each task.}
\label{action_repeat}

\end{table}

\subsection{Results in DeepMind Control}
We note that the previous method DBC \cite{dbc} which utilizes the bisimulation metrics has a side effect on the no-distraction setting in terms of sample efficiency and asymptotic performance. We also evaluate our method in the standard DeepMind Control environments without distractions. We show the results in Figure \ref{fig:stand}. Combined with DrQ, CBM achieves a comparable performance. We guess that DrQ has largely saturated the performance gain in this setting. While combined with DrQ-v2, CBM helps improve the sample efficiency on several tasks.

\subsection{Generalization Experiments}
In the challenging multiple distractions setting, other baselines hardly learn reasonable behaviors, while CBM performs fairly well in almost all the environments. To further evaluate the generalization performance of CBM, we replace the background with 30 unseen videos and evaluate every trained agent. We present the evaluation results in Table \ref{generalization}. The final return is averaged over 50 episodes and the results are reported over 5 seeds. We observe that the performance of CBM only slightly decreases under backgrounds from unseen videos. This result demonstrates that CBM learns robust policies that can generalize under unseen distractions. 
\begin{table*}[t]
	\centering
	\begin{tabular}{ l l l l l l l }
		\toprule
		Method & BiC-Catch & C-Swingup & C-Run & F-Spin & R-Easy & W-walk \\ \midrule
		DrQ + CBM & $705\pm164$ & $529\pm102$ & $243\pm70$ & $856\pm111$ & $336\pm97$ & $30\pm9$ \\
		DrQ-v2 + CBM & ${673\pm172}$ & ${485\pm99}$ & $373\pm82$ & $823\pm103$ & ${410\pm165}$ & ${629\pm119}$\\ 
		\bottomrule
	\end{tabular}
	\caption{Generalization performance of CBM. The polices learned by CBM are able to generalize to unseen environments.}
	\label{generalization}
\end{table*}
\section{Implementation Details}

\subsection{CBM Implementation}
\begin{algorithm}[t]
      \caption{CBM}
      \label{cbm_alg}
    \begin{algorithmic}
        \STATE Initialize the replay buffer $\mathcal{B}\leftarrow\emptyset$
        \STATE Initialize the parameters $\mathbf{\theta}$ of networks
        \STATE Initialize the prototypes and their corresponding rewards
      \FOR{each episode}
      \FOR{each environment step}
      \STATE  Obtain the observation $\mathbf{o}_{t}$
      \STATE Execute actions: $ \mathbf{a}_{t} \sim \pi_\mathbf{\theta}(\cdot| \mathbf{o}_{t+1})$ \
      \STATE Record data: $\mathcal{B}\leftarrow \mathcal{B}\cup\{ \mathbf{o}_t, \mathbf{a}_t, {r}_{t+1}, \mathbf{o}_{t+1} \}$
      \ENDFOR
      \FOR{each training step}
      \STATE Sample a batch of training data
      \STATE Compute RL losses $\mathcal{L}_{\text{RL}}$\
      \STATE Update prototypes' rewards by moving average\
      \STATE Update prototypes' transitions from model $\mathcal{P}$\
      \STATE Compute the bisimulation distance matrix $\mathbf{D}$\
      \STATE Compute the clustering assignment loss $\mathcal{L}_{\text{CBM}}$ \
      \STATE Compute the dynamics loss $\mathcal{L}_{\mathcal{P}}$\
      \STATE $\mathbf{\theta} \leftarrow \mathbf{\theta} - \lambda\nabla_\mathbf{\theta} (\mathcal{L}_{\text{RL}}+\mathcal{L}_{\mathcal{P}}+\mathcal{L}_{\text{CBM}})$
      \ENDFOR
      \ENDFOR
    \end{algorithmic}
    \end{algorithm}

\noindent \textbf{Dynamics Model Optimization}
To approximate the bisimulation distance, we need to train a one-step latent dynamics model. Prior work  \cite{spr,deepmdp} has shown that learning latent dynamics models by minimizing an un-normalized L2 loss over predictions of future latents is prone to collapse. In this paper, we use cosine similarity (equal to a normalized L2 loss) and optimize a contrastive loss \cite{cpc}. Given a batch of data, we encode the observations into latent space, then compute the dynamics loss by
$$
\mathcal{L}_{\mathcal{P}}(\mathbf{\theta})=\sum_{i=1}^B \log \frac{\exp(f(\mathbf{\hat{z}}_i,\mathbf{z}_i^{\prime})/\tau)}{\sum_{k=1}^B \exp
	(f(\mathbf{\hat{z}}_i,\mathbf{z}_k^{\prime})/\tau)}
$$where $\mathbf{\hat{z}}_i=\mathcal{P}(\mathbf{z}_i, \mathbf{a}_i)$ is the prediction of future latent, $\mathbf{z}_i^{\prime}$ is the ground truth of the future latent, and the similarity function $f$ has the form:
$$f(\mathbf{\hat{z}}_i,\mathbf{z}_i^{\prime})=\left(\frac{\mathbf{\hat{z}}_i}{\|\mathbf{\hat{z}}_i\|_2}\right)^{\top}\left(\frac{\mathbf{{z}}_i^{\prime}}{\|\mathbf{{z}}_i^{\prime}\|_2}\right)$$In CBM, a prototype $\mathbf{c}$ can be seen as a point in the latent space. To obtain its one-step transition , we select action $\mathbf{a_c}$ from the current policy and obtain the prototype transition from the learned latent model $\mathbf{c}^{\prime} = \mathcal{P}(\mathbf{c},\mathbf{a_c})$.  

\noindent \textbf{Network Architecture} 
We adopt our network architecture from \citet{sacae}. We implement he shared encoder as a four-layer convolutional network with $3\times3$ kernels and 32 channels. We apply the ReLU activation after each conv layer. The first convolutional layer uses a stride of 2, while the remaining convolutional layers use a stride of 1. 
The output of the convnet is fed into a single fully-connected layer normalized by LayerNorm \cite{ba2016layer}. Finally, we apply tanh nonlinearity to the 50-dimensional output of the fully-connected layer.
The actor and critic modules contain three fully connected layers with hidden dimension 1024. The transition model consists of three fully connected layers with hidden dimension 256.

We run all experiments in one GPU, Geforce 2080Ti. To stabilize the optimization, we block the gradients of the agent's RL loss $\mathcal{L}_{\text{RL}}$ from updating the encoder, prototypes, and the latent model. We illustrate the training process of CBM in Algorithm \ref{cbm_alg}. Moreover, we list the hyperparameters of CBM in Table \ref{hyper}.
\begin{table*}[t]
	\centering
	\begin{tabular}{l c}
		\toprule
		\textbf{Hyperparameter} & \textbf{Setting}\\ 
		\midrule 
		Input dimension & 3$\times$84$\times$84\\ 
		Stacked frames & 3\\ 
		Discount factor & 0.99\\ 
		Episode length & 1000\\ 
		Replay buffer size & 500K\\ 
		Batch size & 128\\ 
		Optimizer & Adam\\ 
		learning rate &5e-4\\ 
		Random cropping padding & 4\\ 
		Seed steps & 4000 \\ 
		SAC entropy temperature & 0.1 \\ 
		Encoder conv layers & 4\\ 
		Encoder conv strides & [2,1,1,1]\\ 
		Encoder conv channels & 32\\ 
		Encoder feature dim & 50\\ 
		Actor head MLP layers & 3\\ 
		Actor head MLP hidden dim & 1024\\ 
		Actor update frequency & 2\\ 
		Critic head MLP layers & 3\\ 
		Critic head MLP hidden dim & 1024\\ 
		Critic target update frequency & 2\\ 
		Critic soft-update rate & 0.01\\ 
		Number of prototypes ($K$)& 128\\
		Prototype reward soft-update rate ($\beta$)& 0.01\\
		Softmax temperature ($\tau$)& 0.1\\
		Sinkhorn regularization parameter ($\varepsilon$) &0.05\\
		DrQ-v2 noise schedule & Cartpole, Finger, Cup, Walker: linear(1.0, 0.1, 100000)\\
		$\quad$ &Cheetah, Reacher: linear(1.0,0.1,500000)\\
		\bottomrule
	\end{tabular}
	\caption{CBM list of hyperparameters. The noise schedule "linear(1.0, 0.1, 500000)" used for DrQv2 means that the
exploration noise decays linearly from 1.0 to 0.1 after 500K environment steps.}\label{hyper}
\end{table*}

\subsection{Baseline Implementation} 
\begin{table*}[ht]
	\centering
	\begin{tabular}{ l l l l l l l }
		\toprule
		Method & BiC-Catch & C-Swingup & C-Run & F-Spin & R-Easy & W-walk \\ \midrule
		DrQ (DCS) & $138\pm20$ & $334\pm29$ & $4\pm2$ & $378\pm125$ & $113\pm22$ & $28\pm1$ \\
		DrQ (Our) & ${99\pm36}$ & ${360\pm36}$ & $203\pm34$ & $545\pm140$ & ${144\pm68}$ & $44\pm14$ \\ 
		\bottomrule
	\end{tabular}
	\caption{Comparison between different implementations in multiple distractions settings. Our implementation achieves similar or better performance than that used in DCS \cite{dcs}.}
	\label{drq_imp}
\end{table*}
We reimplement DrQ \cite{drq} and DrQ-v2 \cite{drqv2} using Pytorch. We made several modifications in the implementations of DrQ to unify the hyperparameters used in DrQ and DrQ-v2. (1)  We use a small learning rate 5e-4 instead of 1e-3. (2) We use a large replay buffer whose size is 500K instead of 100K. To improve the computational efficiency, we use a small batch size 128. Furthermore, we do not use a target encoder. Note that these modifications do not reduce the performance of DrQ and even improve it in some environments. In Table \ref{drq_imp}, we provide a comparison between our implementation and that used in DCS.

For the other three baselines, Proto-RL \cite{proto}, SVEA \cite{svea}, TPC \cite{tpc}, we use their original implementations and default hyperparameters except for adjusting them to the same action repeat setting as in CBM. Additionally, we note that SVEA adopts a deeper convolution network (11 layers). We adjust it to the same four-layer convolution network as in CBM for a fair comparison. 

\section{Additional Visualization Results}\label{appendix:visualization}
We visualize the clustering results of CBM for all the six environments in the multiple distractions setting. Specifically, we sample visual observations (3 consecutive frames) from the replay buffer and obtain their representations using the encoder. In each environment, we visualize the first four prototypes by presenting the three nearest observations to each of them in the latent space. From Figure \ref{fig:walker-clu} to Figure \ref{fig:cup-clu}, we place the three nearest neighbors of a prototype in the same row. We observe that the visual  observations in the same group exhibit similar behaviors. This result demonstrates that CBM maps visual observations with similar behaviors into neighboring points in the latent space and thus achieves task-specific clustering.

\begin{figure*}
\centering
\centerline{\includegraphics[width=1\textwidth]{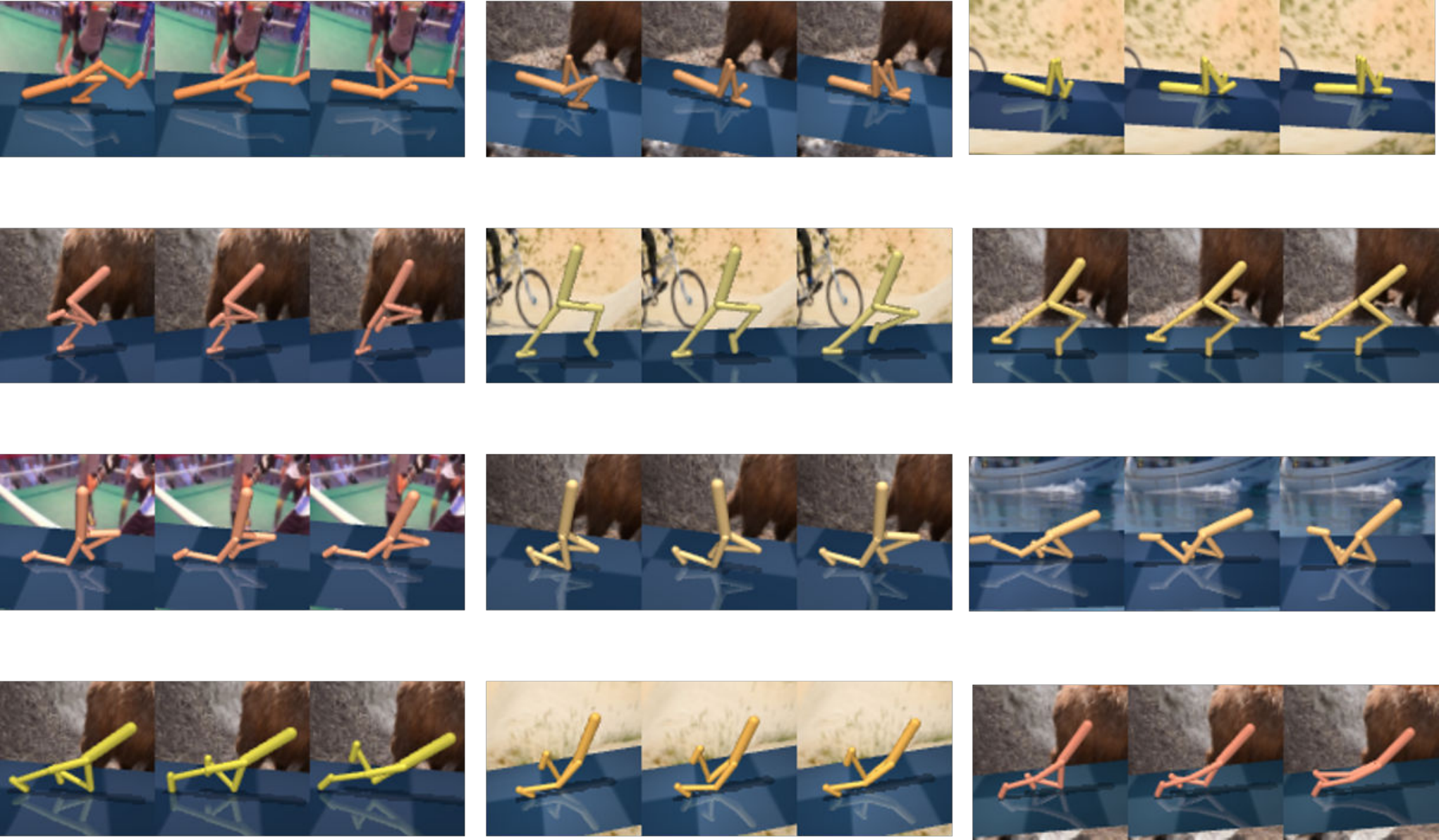}}
\caption{Clustering results visualization in the walker walk task}
\label{fig:walker-clu}
\vskip -0.4cm
\end{figure*}

\begin{figure*}
\centering
\centerline{\includegraphics[width=1\textwidth]{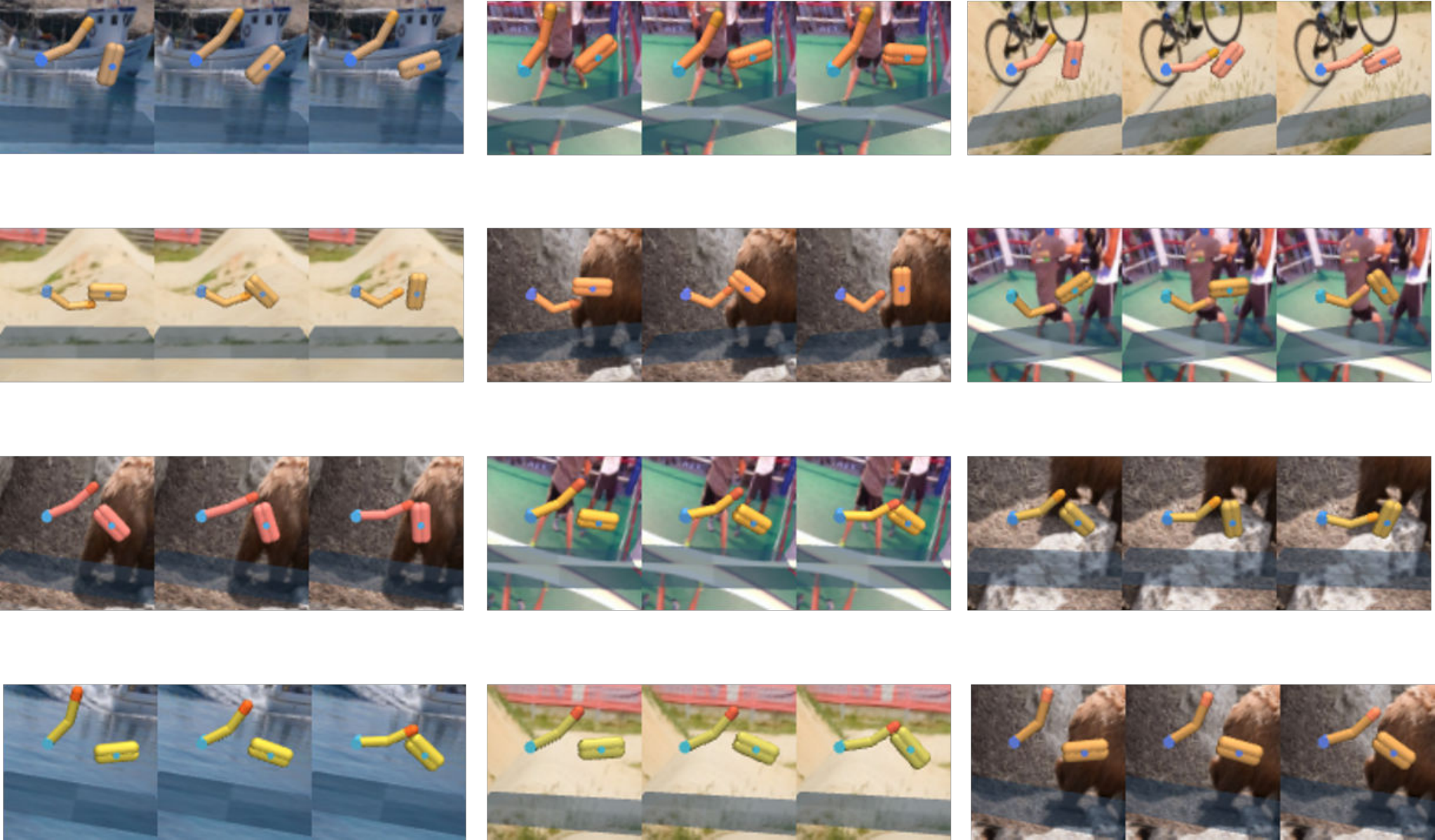}}
\caption{Clustering results visualization in the finger spin task}
\label{fig:finger-clu}
\vskip -0.4cm
\end{figure*}
\begin{figure*}
\centering
\centerline{\includegraphics[width=1\textwidth]{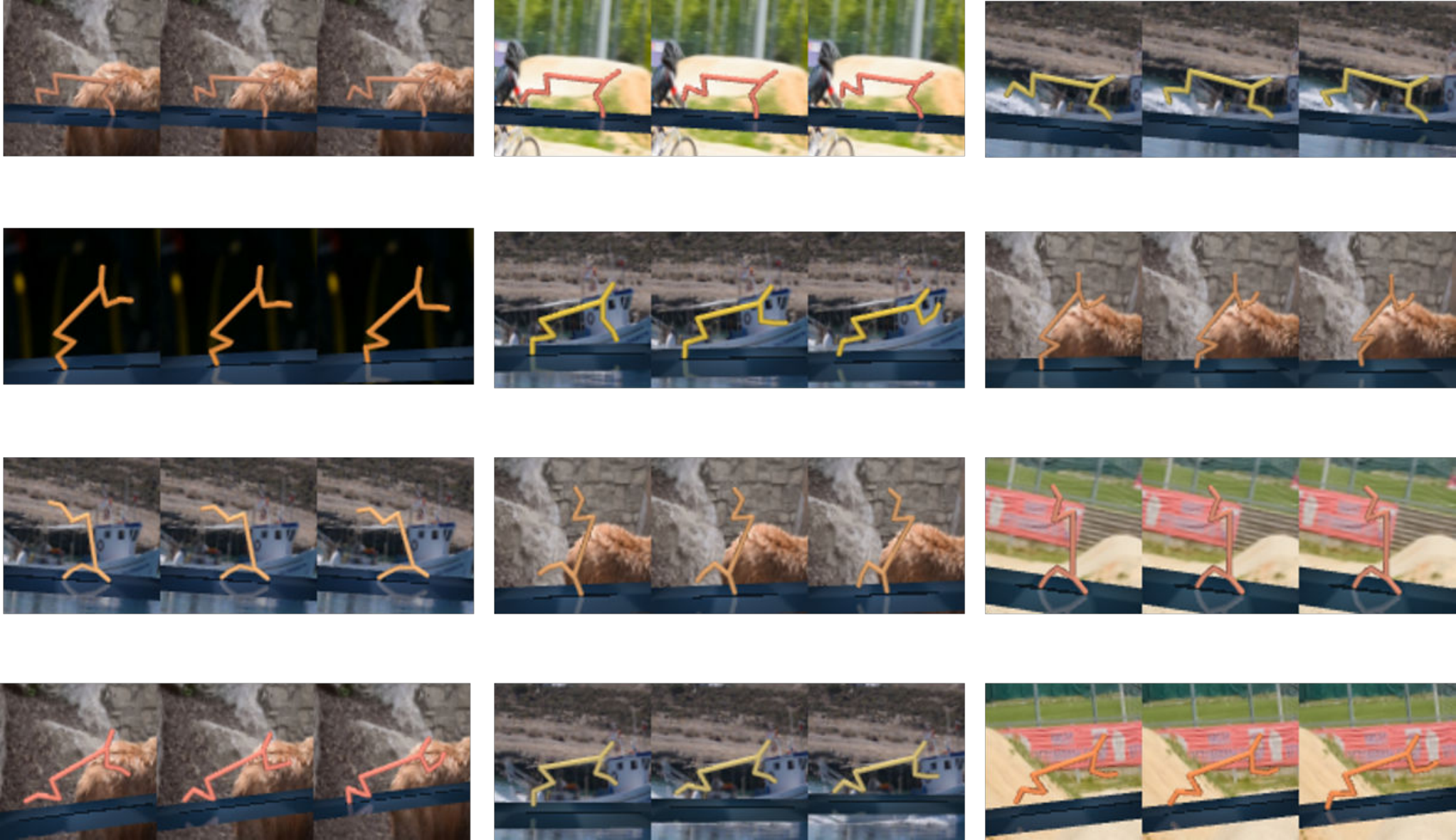}}
\caption{Clustering results visualization in the cheetah run task}
\label{fig:cheetah-clu}
\vskip -0.4cm
\end{figure*}
\begin{figure*}
\centering
\centerline{\includegraphics[width=1\textwidth]{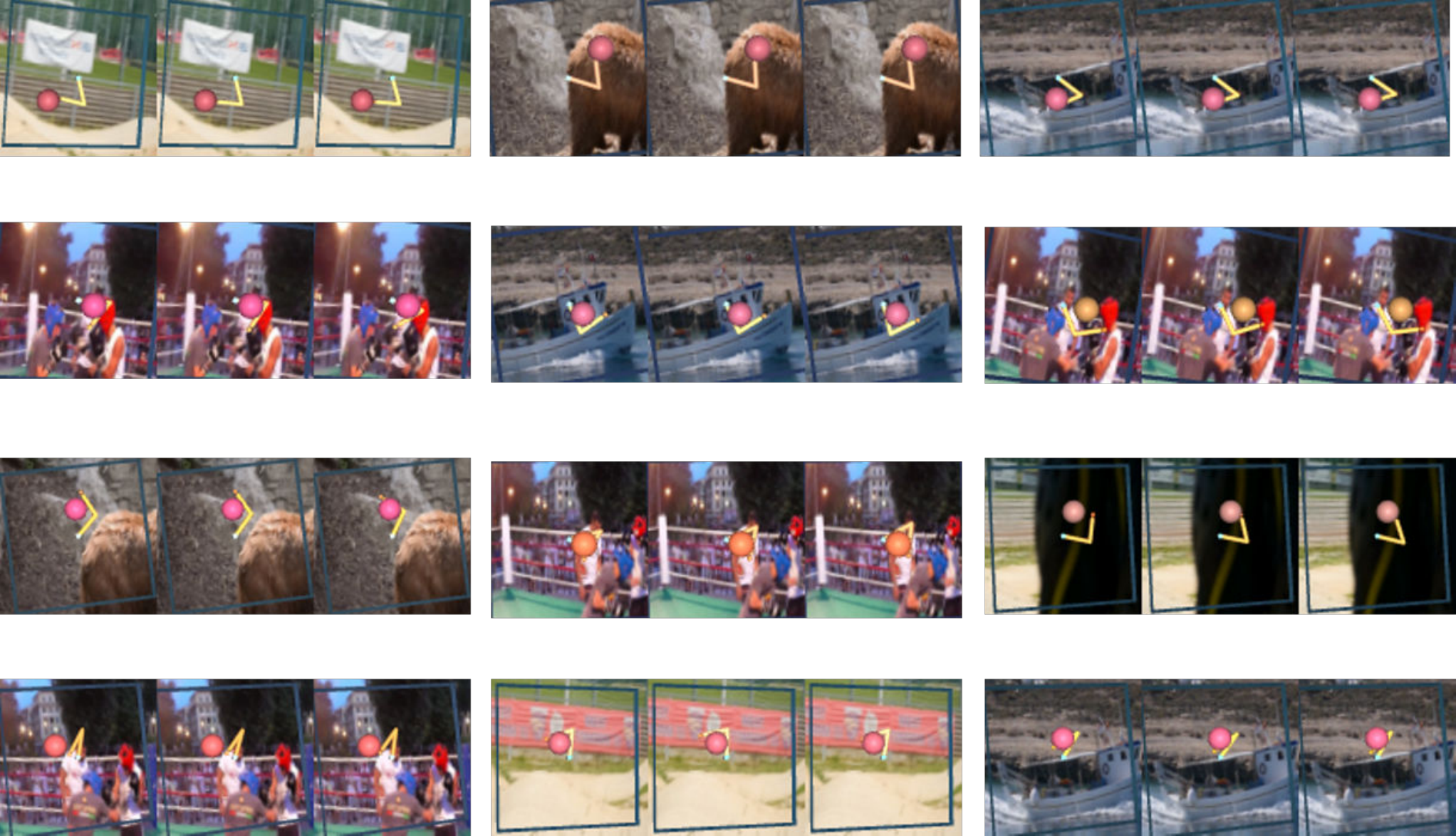}}
\caption{Clustering results visualization in the reacher easy task}
\label{fig:reacher-clu}
\vskip -0.4cm
\end{figure*}
\begin{figure*}
\centering
\centerline{\includegraphics[width=1\textwidth]{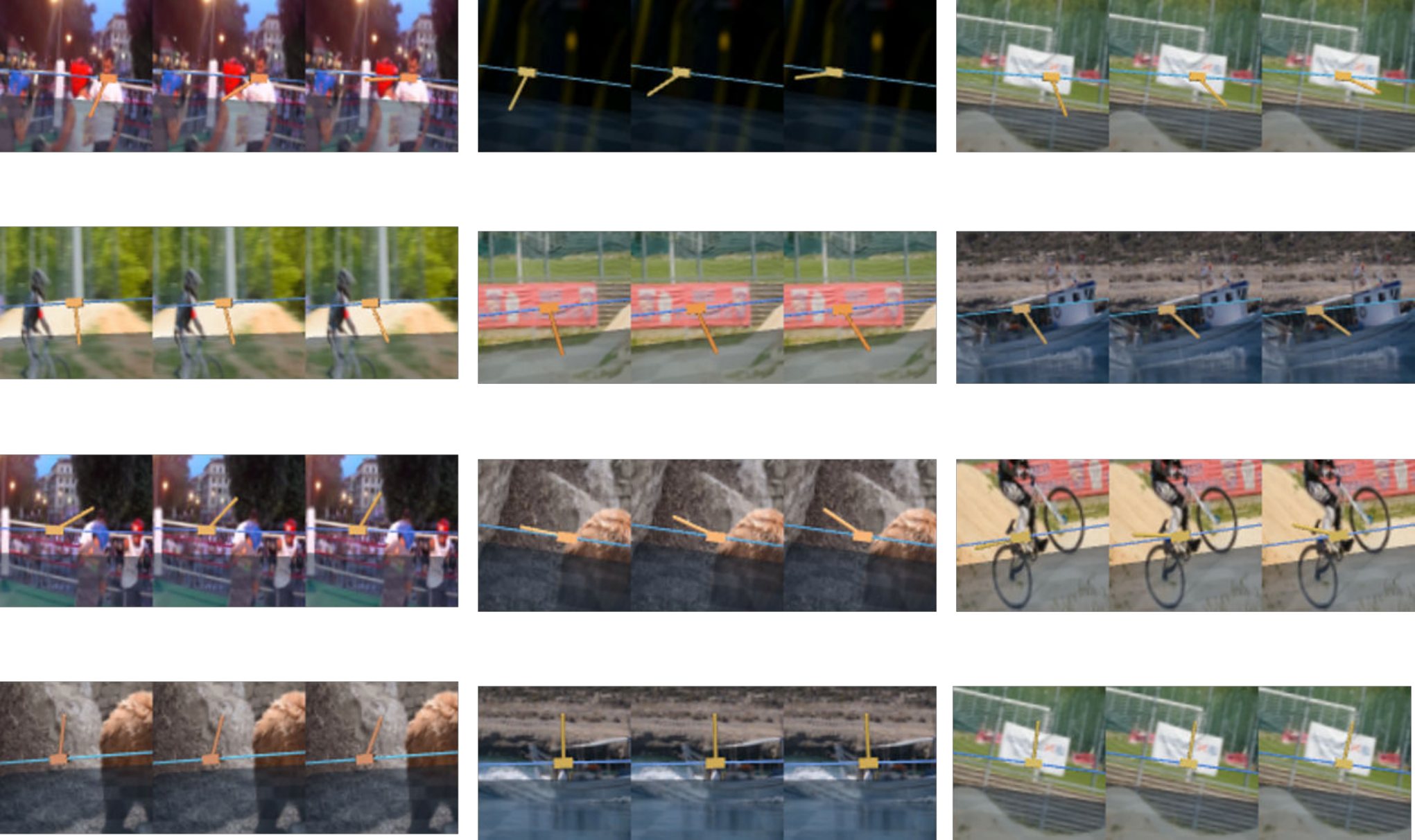}}
\caption{Clustering results visualization in the cartpole swingup task}
\label{fig:cartpole-clu}
\vskip -0.4cm
\end{figure*}
\begin{figure*}
\centering
\centerline{\includegraphics[width=1\textwidth]{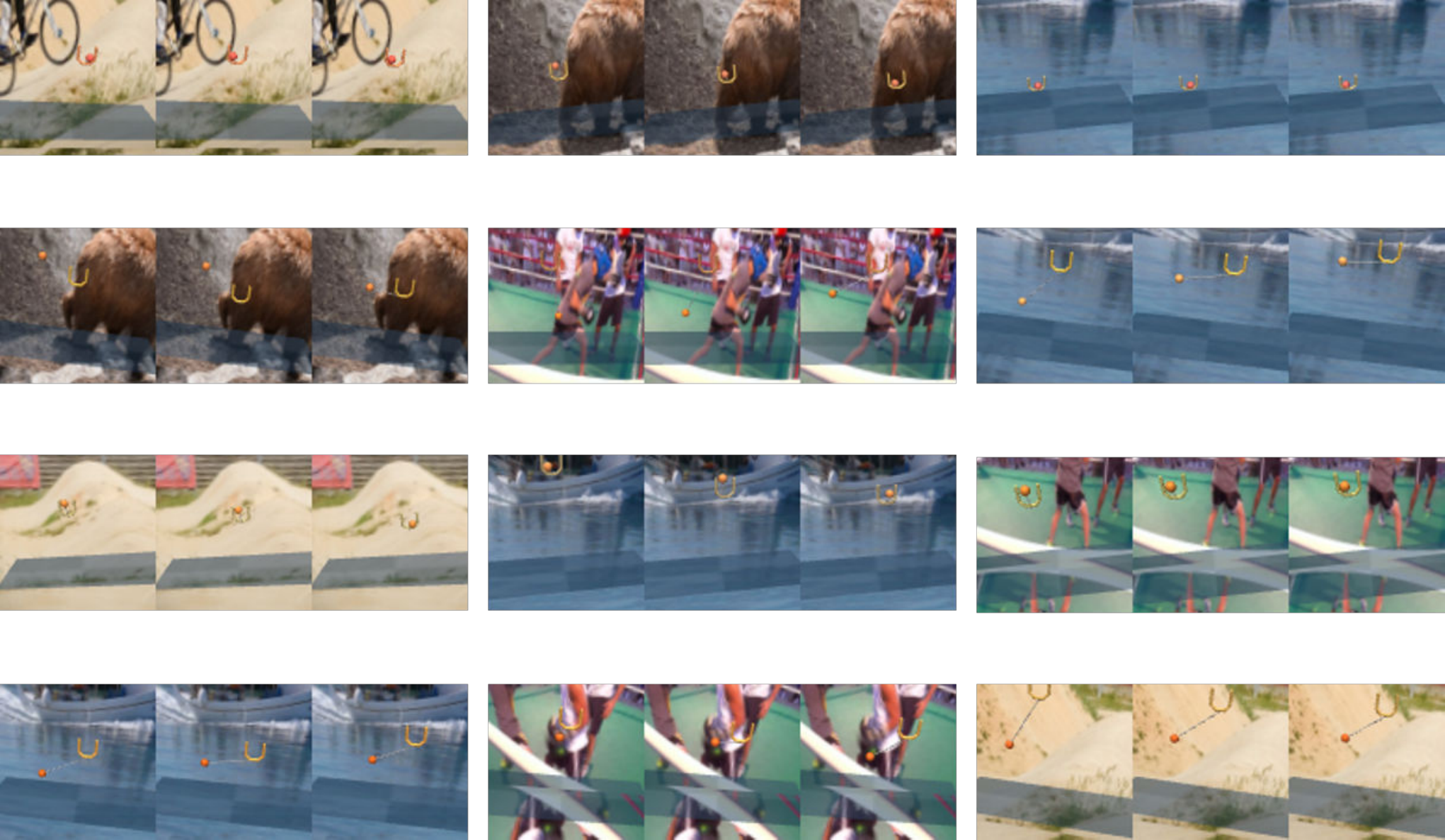}}
\caption{Clustering results visualization in the ball in cup catch task}
\label{fig:cup-clu}
\vskip -0.4cm
\end{figure*}
\end{document}